\documentclass{article}

\PassOptionsToPackage{numbers,compress}{natbib}

\usepackage[preprint]{style/neurips_2025}

\usepackage[utf8]{inputenc} 
\usepackage[T1]{fontenc}    
\usepackage{hyperref}       
\usepackage{url}            
\usepackage{booktabs}       
\usepackage{amsfonts}       
\usepackage{nicefrac}       
\usepackage{microtype}      
\usepackage{xcolor}         

\usepackage{booktabs}       
\usepackage{amsfonts}       
\usepackage{nicefrac}       
\usepackage{microtype}      

\usepackage{stmaryrd} 
\usepackage{amsthm}

\usepackage{thmtools, thm-restate} 
\usepackage[normalem]{ulem} 
\usepackage{mdframed} 
\usepackage{multirow} 
\usepackage{amsmath}
\usepackage{amssymb} 
\usepackage{nccmath} 

\usepackage{float}
\usepackage{subcaption} 

\usepackage{colortbl} 

\usepackage{vwcol} 

\usepackage{pifont} 

\usepackage{multicol}

\usepackage{tabularx}

\usepackage{graphicx}
\usepackage[export]{adjustbox} 
\usepackage{graphbox} 
\usepackage{array, makecell, multirow}

\usepackage{enumitem}

\usepackage{wrapfig}
\usepackage{tocloft}

\usepackage{mathtools} 
\usepackage{relsize} 

\usepackage{bbold} 

\usepackage{xargs} 

\usepackage{framed}

\usepackage{hyperref}
\hypersetup{
    colorlinks=true,
    linkcolor=blue,
    urlcolor=blue,
    citecolor=blue,
}

\usepackage{cleveref} 

\newtheorem{theorem}{Theorem}
\newtheorem{lemma}{Lemma} 
 
\newtheorem{remark}{Remark}
\newtheorem{corollary}{Corollary}
\newtheorem{definition}{Definition}

\newtheorem{property}{Property}

\newtheorem*{theorem*}{Theorem}

\newcommand{\Lojasiewicz}{{\L}ojasiewicz}

\newtheorem{assumption}{Assumption}

\crefname{algo}{Algorithm}{Algorithms}
\crefname{model}{Model}{Models}
\crefname{lemma}{Lemma}{Lemmas}
\crefname{fact}{Fact}{Facts}
\crefname{theorem}{Theorem}{Theorems}
\crefname{corollary}{Corollary}{Corollaries}
\crefname{claim}{Claim}{Claims}
\crefname{example}{Example}{Examples}
\crefname{problem}{Problem}{Problems}
\crefname{definition}{Definition}{Definitions}
\crefname{assumption}{Assumption}{Assumptions}

\crefname{subsection}{Subsection}{Subsections}
\crefname{section}{Section}{Sections}
\crefname{algorithm}{Algorithm}{Algorithms}
\crefname{algocf}{alg.}{algs.}
\Crefname{algocf}{Algorithm}{Algorithms}
\crefname{proposition}{Proposition}{Propositions}
\crefname{exemple}{Exemple}{Examples}
\crefname{remark}{Remark}{Remarks}
\crefname{property}{Property}{Properties}
\crefname{inequality}{Inequality}{Inequalities}

\crefname{goal}{Goal}{Goals}
\crefname{application}{Application}{Applications}

\newcommand{\gs}{\vspace{-0em}}

\DeclarePairedDelimiter\floor{\lfloor}{\rfloor}

\newcommand{\ffrac}[2]{\ensuremath{\frac{\displaystyle #1}{\displaystyle #2}}}

\newcommand{\allforone}{\texttt{All-for-one}}

\newcommand{\FedAvg}{\texttt{FedAvg}}

\newcommand{\E}{\mathbb{E}}
\newcommand{\N}{\mathbb{N}}
\newcommand{\R}{\mathbb{R}}

\newcommand{\V}{\mathbb{V}}
\newcommand{\lnrm}{\left \|} 
\newcommand{\rnrm}{\right \|} 

\newcommand{\vertiii}[1]{{\vert\kern-0.25ex\vert\kern-0.25ex\vert #1 
    \vert\kern-0.25ex\vert\kern-0.25ex\vert}}

\newcommand{\FullExpec}[1]{\E \left[#1\right]} 
\newcommand{\fullexpec}[1]{\E [#1]} 

\newcommand{\Expec}[2]{\E \left[#1~\middle|~#2\right]} 
\newcommand{\expec}[2]{\E [#1~|~#2]} 

\newcommand{\PdtScl}[2]{\left\langle#1,#2\right\rangle}  
\newcommand{\pdtscl}[2]{\langle#1,#2\rangle}  
\newcommand{\SqrdNrm}[1]{ \lnrm #1\rnrm^2} 
\newcommand{\sqrdnrm}[1]{ \| #1\|^2} 
\newcommand{\bigpar}[1]{\left( #1 \right)} 

\newcommand{\OneToN}{[N]}
\newcommand{\OneToT}{[T]}
\newcommand{\OneToK}{[K]}

\newcommand{\smoothness}{\beta}

\DeclareFontFamily{U}{wncy}{}
\DeclareFontShape{U}{wncy}{m}{n}{<->wncyr10}{}
\DeclareSymbolFont{mcy}{U}{wncy}{m}{n}
\DeclareMathSymbol{\Sha}{\mathord}{mcy}{"58}

\usepackage{tabto}

\makeatletter
\newcommand{\oset}[3][0ex]{%
  \mathrel{\mathop{#3}\limits^{
    \vbox to#1{\kern-2\ex@
    \hbox{$\scriptstyle#2$}\vss}}}}
\makeatother

\definecolor{brickred}{rgb}{0.8, 0.25, 0.33}

\definecolor{tabblue}{HTML}{1F77B4}
\definecolor{taborange}{HTML}{FF7F0e}
\definecolor{tabred}{HTML}{d62728}
\definecolor{tabgreen}{HTML}{2ca02c}
\definecolor{tabgray}{HTML}{7f7f7f}

\definecolor{forestgreen}{rgb}{0.13, 0.55, 0.13}

\definecolor{carmine}{rgb}{0.59, 0.0, 0.09}

\newcommand{\tcmv}[1]{\textcolor{tabgreen}{#1}}

\usepackage{pgf}

\definecolor{lowNegative}{rgb}{1,0.8,0.8} 
\definecolor{mediumNegative}{rgb}{1,1,0.5} 
\definecolor{highNegative}{rgb}{0.8,1,0.8} 

\definecolor{lowPositive}{rgb}{0.8,1,0.8} 
\definecolor{mediumPositive}{rgb}{1,1,0.5} 
\definecolor{lightred}{rgb}{1,0.8,0.8} 

\newcommand{\applyColor}[1]{
  \pgfmathparse{#1}
  \let\value\pgfmathresult
  \ifdim \value pt > 0.05pt
    \cellcolor{tabgreen}#1
  \else\ifdim \value pt > 0.01pt
    \cellcolor{lightred}#1
  \else
    \cellcolor{tabred}#1
  \fi\fi
}
\defcitealias{krizhevsky_learning_2009}{Kri09}
\defcitealias{caruana_kdd-cup_2004}{CTL04}
\defcitealias{lecun_gradient-based_1998}{LeCun98}
\defcitealias{terrail2022flamby}{Ter22}

\title{Adaptive collaboration for online personalized distributed learning with heterogeneous clients}

\author{%
  Constantin Philippenko$^1$, Batiste Le Bars$^2$, Kevin Scaman$^1$, Laurent Massoulié$^1$ \\
  $^1$Argo Team, Inria Paris - Département d’informatique de l’ENS, PSL Research University\\
  (firstname).(lastname)@inria.fr \\
  $^2$Univ. Lille, Inria, CNRS, Centrale Lille, UMR 9189, CRIStAL, F-59000 Lille\\
  batiste.le-bars@inria.fr \\
}

\begin{document}

\maketitle


\begin{abstract}
	We study the problem of online personalized decentralized learning with $N$ statistically heterogeneous clients collaborating to accelerate local training. An important challenge in this setting is to select relevant collaborators to reduce gradient variance while mitigating the introduced bias. To tackle this, we introduce a gradient-based collaboration criterion, allowing each client to dynamically select peers with similar gradients during the optimization process. Our criterion is motivated by a refined and more general theoretical analysis of the \texttt{All-for-one} algorithm, proved to be optimal in Even et al. (2022) for an oracle collaboration scheme. We derive excess loss upper-bounds for smooth objective functions, being either strongly convex, non-convex, or satisfying the Polyak-Łojasiewicz condition; our analysis reveals that the algorithm acts as a variance reduction method where the speed-up depends on a \emph{sufficient variance}. 
	We put forward two collaboration methods instantiating the proposed general schema; and we show that one variant preserves the optimality of \texttt{All-for-one}.  We validate our results with experiments on synthetic and real datasets.
\end{abstract}

\section{Introduction}

Distributed learning has emerged as a popular paradigm where multiple clients collaboratively train machine learning models without sharing raw data \citep{konecny_federated_2016,mcmahan_communication-efficient_2017}. In statistically heterogeneous environments, personalization becomes crucial as clients often have non-identically distributed data and distinct learning objectives \citep{hanzely2020lower,mansour2020three,tan2022towards}. In such heterogeneous settings, the classical \FedAvg~algorithm fails to find a satisfying global optimal points \citep{li2019convergence}. A key challenge is therefore to design collaboration strategies that selectively exploit similarities across clients to improve the train/test performance \citep{beaussart2021waffle,chayti2021linear,even2022sample}. 

Formally, we consider a distributed setting with $N \in \N$ clients, where each client $i \in \OneToN$ wants to find a model $\theta^\star_i$ in $\R^d$ ($d$ the number of features) minimizing the local loss $R_{i}: \R^d \mapsto \R$. Hence, we do not seek to find a global minimizer of the averaged loss as done for instance with \FedAvg, but instead to find $N$ personalized models. We focus on stochastic gradient descent-like algorithms \citep{robbins1951stochastic} in an online setting where at each iterations $t$ in $\N$, every client has access to an unbiased stochastic gradient oracle  $g_i^t$.
In order to reduce the sample complexity, i.e. the number of samples or stochastic gradients required to reach small excess loss error, clients may have the possibility to use stochastic gradients from other clients, which can however 
be biased since in general for two clients $i,k$ in $\OneToN$, $\E g_i^t \neq \nabla R_{k}$ \citep{karimireddy2020scaffold,chayti2021linear}.

In this work, we consider a class of algorithms called \emph{weighted averaging aggregation rules} \citep[WGA, e.g.][]{chayti2021linear}. More precisely, for a specific client $i\in\OneToN$, a set of weights $\alpha_{ik}^{t}$ in $[0, 1]$ is computed, either on the client itself, or on a third-party server, before proceeding to the following update:
\begin{align}
	\label{eq:update_equation}\tag{\allforone}
	\forall i \in \OneToN, \forall t \in \R^\star, \quad \theta^t_i = \theta^{t-1}_i - \eta_i^t \sum_{k=1}^N \alpha_{ik}^{t}\, g_k^t(\theta^{t-1}_i) \,.
\end{align}

This rather simple algorithm is referred as \texttt{All-for-one} in \citet{even2022sample} since all clients agree to compute stochastic gradients at the parameter $\theta_i$ of any other client $i$. Notice that for generality, we consider that the aggregation parameter $\alpha$ may change over time -- leading to dynamically build clusters -- which was not considered in the original paper. It is important to also notice that the collaboration here makes sense because of the stochastic nature of the gradients. Indeed, if each client had access to the full gradient at each iteration, there would be no point in collaborating. Clients want to collaborate in order to speed-up their training by reducing their variance, this is why \allforone~can be seen as a \textit{variance reduction} method compared to full local training. In addition, appropriately selecting the weights also controls the bias introduced by other clients, thereby it can be seen as a \emph{bias reduction} mechanism compared to \FedAvg.

For a certain class of losses with uniformly bounded gradient difference, \citet{even2022sample} derive a lower bound on the number of iterations required to reach a certain optimization error $\varepsilon > 0$. They further show that, for any $\varepsilon > 0$, there exists a set of weights $(\alpha_{ik})_{i \in \OneToN, k \in \OneToN}$ that makes \allforone~optimal in term of sample complexity. However, fixing $\alpha$ in their setting requires to know, \emph{before running the algorithm}, the uniform upper-bound on the clients (gradients) difference, reason why it is referred to as \emph{oracle} in this work. Moreover, $\alpha$ also depends on the final desired accuracy $\varepsilon$ -- which makes it even more unpractical.

This work seeks to address these limitations. First, we aim to derive the weights $\alpha$ in an \emph{online manner} -- i.e. during the optimization process -- while preserving the optimality.  Second, we seek convergence guarantees and weight update mechanisms that do not depend on the final target accuracy $\varepsilon$, thus enhancing the algorithm’s practicality and applicability in online and dynamic settings. 

\paragraph{Related works.} Statistical heterogeneity is a classical scenario in federated/decentralized learning. A standard approach consists in learning a single global model that minimizes the weighted average of local objective functions across clients \citep[e.g.][]{li2019convergence, karimireddy2020scaffold, li2020federated}. However, this global model may fail to capture client-specific patterns, leading to the growing interest in personalized federated learning (PFL) \citep{kulkarni2020survey, sattler2020clustered, dinh2020personalized, tan2022towards, capitaine2024unravelling}. A central question in PFL is how to effectively orchestrate client collaboration to balance personalization and generalization. Existing approaches address this by: clustering clients into groups with similar data distributions \citep{ghosh2020efficient, muhammad2020fedfast, sattler2020clustered, werner2023provably}; modeling local distributions as mixtures of unknown latent distributions \citep{marfoq2021federated}; performing global training followed by local fine-tuning \citep{dinh2020personalized, fallah2020personalized}; learning a shared representation complemented by local personalized modules \citep{arivazhagan2019federated, collins2021exploiting, liang2020think, philippenko2025depth}; formulating the learning problem as a bilevel optimization task \citep{deng2020adaptive, hashemi2024cobo}; combining global and local objectives via regularization or interpolation \citep{deng2020adaptive, li2021ditto, mansour2020three, liu2023feddwa}; or maintaining $N$ local models and updating them using weighted gradient averaging schemes \citep{beaussart2021waffle,chayti2021linear,even2022sample}.

\paragraph{Contributions.} This work follows the later approach and aims to design a practical algorithm strengthened by guarantees of convergence. Our contributions can be summarized as follows.
\begin{itemize}[leftmargin=*, itemsep=0pt, parsep=0pt, partopsep=0pt]
	\item Driven by practical considerations, we present a refined convergence analysis of the \allforone~algorithm, where collaboration weights are dynamically updated over time. Our analysis reveals a key relationship between client heterogeneity, gradient similarity at each iteration, and the resulting collaboration weights.
	\item More precisely, we introduce a general and theoretical framework for clients collaboration. We derive excess loss upper-bounds for smooth objective functions under three regimes: strongly convex, non-convex, and satisfying the Polyak–Łojasiewicz condition, validating the relevance of our method.
	\item We put forward two collaboration methods instantiating the proposed general schema -- one based on binary collaboration weights, and the other on continuous ones -- and explain how to to compute the weights in practice. We show that the binary schema still preserve the sampling optimality of the original \allforone~algorithm, while removing its dependence on the target accuracy $\varepsilon$.
	\item  Finally, we support our theory with experiments on synthetic and real datasets, showcasing the empirical performance of our schemes.
\end{itemize}

The rest of the paper is organized as follows. In \Cref{sec:math_background}, we present the mathematical background and introduce a descent lemma that serves as the starting point of our analysis; in \Cref{sec:adaptive_collaboration} we present the general framework for collaboration along with convergence guarantees; in \Cref{sec:application} we propose two choices of practical collaboration criterion; in \Cref{sec:experiments} we illustrate our analysis with experiments, \Cref{sec:ccl} concludes the paper, discusses broader impacts, and outlines open directions stemming from the limitations of our analysis.

\section{Mathematical background: a key descent lemma}
\label{sec:math_background}

In this section, we describe precisely the mathematical tools necessary to develop our analysis, and present our first key result. We denote by $\E$ and $\V$ the expectation and variance. We define for any client $i$ in $\OneToN$ and for any iteration $t$ in $\N^\star$, the local excess loss of client $i$ at time $t$ as $\varepsilon_{i}^t = \FullExpec{R_{i}(\theta_i^t) - R_{i}(\theta^\star_i)}$, where we assume that there exist at least one global optimum $\theta^\star_i \in \mathrm{argmin}_\theta R_{i}(\theta)$. We denote  $\mathcal{N}_i^t = \{k\in[N]~:~\alpha_{ik}^t > 0\}$ the set of clients collaborating with client $i$ at iteration $t$. We consider the following classical assumptions on the loss functions $(R_i)_{i \in \OneToN}$.

\begin{assumption}[$\smoothness$-smoothness] 
	\label{asu:smoothness}
	For any client $i$ in $\OneToN$, we assume that the loss function $R_i$ is $\smoothness$-smooth, i.e., $\exists \smoothness > 0$ such that $\forall \theta, \theta' \in \R^d, \|\nabla R_i(\theta) -
	\nabla R_i(\theta')\| \leq \smoothness\| \theta - \theta'\|$.
\end{assumption}

\begin{assumption}[Uniformly Bounded Stochastic Gradient Variance]
	\label{asu:bounded_variance}
	At any iteration $t \in \N$, the variance of the stochastic gradient is uniformly bounded for each local objective functions $(R_{i})_{i \in \OneToN}$, i.e., for any client $i$ in $\OneToN$, there exists a constant \(\sigma_i^2 > 0\) such that for all \(\theta \in \mathbb{R}^d\):
	$
	\V[g_{i}^t(\theta)] \leq \sigma_i^2,
	$
	where $g_{i}^t$ is the stochastic gradient at iteration $t$.
\end{assumption}

Based on these two standard assumptions in optimization, we have the following descent lemma, backbone of our analysis, that allows to derive all the other theorems \emph{via} the collaboration criterion.
\begin{lemma}[Descent lemma]
	\label{lem:descent_lemma}
	Let \allforone~be run on client $i$ in $\OneToN$ for $t$ in $\N$ iterations. Under \Cref{asu:bounded_variance,asu:smoothness}, taking for all $k$ in $\OneToN$, $\alpha_{ik}^t \geq 0 $ and $\eta_i^t < (\smoothness \sum_{k \in \OneToN} \alpha_{ik})^{-1}$, we have: 
	\begin{align*}
		\expec{\varepsilon_{i}^t}{\theta^{t-1}_i} - \varepsilon_{i}^{t-1}  &\leq \frac{\eta_i^t}{2}  \sum_{k\in \OneToN} \alpha_{ik}^{t} \bigpar{\sqrdnrm{\nabla R_{k}(\theta^{t-1}_i) - \nabla R_{i}(\theta^{t-1}_i)} -\sqrdnrm{\nabla R_{i}(\theta^{t}_i)}} +  \eta_i^t\beta \sigma^2_k {\alpha_{ik}^{t}}^2  
	\end{align*}
\end{lemma}

The proof is provided in \Cref{proof:descent_lemma}.
Notice that our analyses does not require to have weights summing to $1$. This descent lemma results to bound the excess loss of client $i$, in line with \citep{chayti2021linear,even2022sample}, rather than the average excess loss across clients as in \citep{li2021ditto,deng2020adaptive,hashemi2024cobo}.

The core of our analysis lies in the appearance of the term $\|\nabla R_{k}(\theta^{t-1}_i) - \nabla R_{i}(\theta^{t-1}_i)\|^2$. Intuitively, collaboration should occur only when this discrepancy is small compared to $\|\nabla R_i(\theta^{t-1}_i)\|^2$; otherwise, client $i$ should not incorporate client $k$'s update and set $\alpha_{ik} = 0$. From this intuition flows the collaboration criterion formalized in the next section.  The final term corresponds to the gradient variance, which is standard but is here reduced by the collaboration weights $\alpha_{ik}^t$.

To allow collaboration, we deduce from \Cref{lem:descent_lemma} that the norm of the pairwise gradients difference must remain sufficiently small. The following bounded heterogeneity assumption ensures a control over this difference and aligns with the collaboration criterion proposed in \Cref{sec:adaptive_collaboration}.

\begin{assumption}[Bounded Heterogeneity]
	\label{asu:bounded_heterogeneity}
	For each couple of clients $i$ and $k$ in $\{1,\ldots,N\}$, there exist $b_{ik}\geq 0$ and $c_{ik}\geq 0$ s.t. for all $\theta$ in $\R^d$: $\SqrdNrm{ \nabla R_{i}(\theta)- \nabla R_{k}(\theta)}\leq b_{ik}^2 + c_{ik}\SqrdNrm{\nabla R_{i}(\theta)}$\;. 
\end{assumption}

This assumption quantifies the gradient dissimilarity between clients: the deviation between $\nabla R_i$ and $\nabla R_k$ is bounded by an additive bias term $b_{ik}^2$, that captures gradient misalignment (e.g., due to data distribution shift), and a multiplicative term $c_{ik} \SqrdNrm{\nabla R_i(\theta)}$, which scales with the local gradient magnitude. This affine formulation generalizes prior uniformly-bounded assumptions ($c_{ik}=0$) such as in \citet{even2022sample}. \citet{chayti2021linear} also consider it, however, they require an additional assumption characterizing client dissimilarity to remove the resulting bias (see their Theorems 4.1 and 5.1), which is not needed in our analysis. A similar distinction in two cases --  whether clients are collaborative or not --- is also made in \citet{hashemi2024cobo,werner2023provably}. These two articles further assume the existence of latent clusters of clients sharing identical stationary points, which is something we don't do. Below, we explicit \Cref{asu:bounded_heterogeneity} in the case of quadratic functions.

\begin{property}[Heterogeneity assumption in the case of quadratic functions]
	\label{prop:quadratic_fn}
	For any $i$ in $\OneToN$, in the case of quadratic function of the form $R_{i} : \theta \mapsto (\theta + \xi_i)^T A_i (\theta + \xi_i)$, where $\xi_i$ in $\R^d$ and $A_i$ in $\R^{d\times d}$, we have explicit bounds to quantify heterogeneity:
	\begin{align*}
		b_{ik}  \leq \sqrt{2} \|A_k ( \xi_i - \xi_k) \|_2 \qquad \text{and} \qquad c_{ik} \leq 2 \sqrdnrm{I - A_k A_i^{-1}}_2 \,.
	\end{align*}
	
\end{property}

This property illustrates the interest of our assumption over the uniformly bounded one ($c_{ik}=0$) from \citet{even2022sample}: two different quadratic functions, even if they share the same minimum, do not have uniformly bounded gradient difference; they however satisfy our assumption with $b_{ik} = 0$. 

Finally, we consider either the strongly-convex or the Polyak-\Lojasiewicz~settings, except in \Cref{thm:convergence_non_cvx} where we only consider the smooth non-convex scenario.

\begin{assumption}[Strong convexity]
	\label{asu:strongly_convex}
	For each client $i$ in $\OneToN$, $R_{i}$ is $\mu$-strongly convex, that is for all vectors $\theta, \theta'$ in $\R^d$:
	$R_{i}(\theta') \geq R_{i}(\theta) + (\theta' -\theta)^T \nabla R_{i}(\theta) + \frac{\mu}{2} \| \theta' - \theta \|^2_2\,.$
	\gs
\end{assumption}

\begin{assumption}[$\mu$-Polyak-\Lojasiewicz~(PL) condition]
	\label{asu:PL}
	For each client $i$ in $\OneToN$, $R_{i}$ satisfies the $\mu$-Polyak-\Lojasiewicz~(PL) condition with $\mu > 0$, if for all $\theta \in \R^d$, $\frac{1}{2} \|\nabla R_{i}(\theta)\|^2 \geq \mu \bigpar{R_{i}(\theta) - R_{i}(\theta_i^\star)}$ 
	where $\theta^\star_i \in \mathrm{argmin}_\theta R_{i}(\theta)$.
\end{assumption}

\section{Adaptive collaboration and guarantees of convergence}
\label{sec:adaptive_collaboration}

In this section, we provide a class of collaboration weights that allow for efficient collaboration and simple theoretical analysis.
In particular, optimizing the upper bound in \Cref{lem:descent_lemma} shows that collaboration weights should be zero if $\sqrdnrm{\nabla R_{k}(\theta^{t-1}_i) - \nabla R_{i}(\theta^{t-1}_i)}\geq \sqrdnrm{\nabla R_{i}}$, and proportional to $(\sqrdnrm{\nabla R_{i})} - \sqrdnrm{\nabla R_{k}(\theta^{t-1}_i) - \nabla R_{i}(\theta^{t-1}_i)})/\sigma_k^2$ otherwise. We thus denote as \emph{similarity ratio} at time $t$ the quantity $r_{ik}^t :=\bigpar{ 1 -\frac{\sqrdnrm{\nabla R_{k}(\theta^{t-1}_i) - \nabla R_{i}(\theta^{t-1}_i)}}{\sqrdnrm{\nabla R_{i}(\theta^{t-1}_i)}}}_+$, and consider the following weights.

\begin{definition}[Collaboration weights]
	\label{def:weights_binary} 
	Let $\phi: [0, 1]  \mapsto [0, 1]$ be non-decreasing \emph{criterion function} s.t. $\phi(0) = 0$ and $\phi(x) \leq x$. Then, we consider the weights~$
	\alpha_{ik}^{t} = \phi(r_{ik}^t) (\sigma_{i, \psi}^t / \sigma_k)^2$, 
	where $\psi : x \mapsto x \phi(x)$, and for any function $f$, we denote $\sigma^t_{i, f} \coloneqq ( \sum_{k=1}^N \sigma_k^{-2} f(r_{ik}) )^{-1/2}$.
\end{definition}

While the formulation of $\alpha_{ik}$ may initially appear complex, its particular form will allow for simple upper bounds (see \Cref{thm:descent_lemma_with_N_i_t}). It recovers several existing settings as special cases and satisfies desirable theoretical properties. Further details are provided in \Cref{sec:application} where we will instantiate this general framework with two explicit choices for $\phi$, namely a \emph{binary} and a \emph{continuous} criterion. 

To give some intuition on the above definition, observe that the weights are an increasing function of the similarity ratio, and inversely proportional to the gradient variance. This comes from the fact that clients with low gradient variance $\sigma_k$ can be seen as more reliable and thus receive higher weights. The similarity ratio $r_{ik}$ measures how far client $k$'s gradient is from client $i$'s gradient, relatively to $i$'s own gradient norm. If this normalized squared distance is small, then client $k$ is considered informative and close enough to help client $i$, otherwise, $\alpha_{ik}^t=0$, meaning no collaboration. Notice that as the loss of client $i$ gets better optimized, its gradient norm decreases and less collaborations are made. 

\begin{remark}
	\Cref{def:weights_binary} requires to know the variance of each clients, however, assume that $\sigma_k = \sigma / \sqrt{n_k}$, $n_k$ being the batch size of client $k$ at each iteration, then $ (\sigma_{i, \psi}^t / \sigma_k)^2 $ is simply equal to $(\frac{1}{n_k} \sum_{j \in \mathcal{N}_i^t} n_j \psi(r_{ij}))^{-1}$ which allows in practice to compute the weights while taking into account the clients' size heterogeneity. It also means that among similar clients, those with more stable gradients contribute more heavily, improving variance reduction. 
\end{remark}

Based on this collaboration criterion and on descent \Cref{lem:descent_lemma}, we derive the following convergence bound on the excess loss at time $T$ in $\N^\star$. The proof is given in \Cref{app:subsec:proof_cvgce}.

\begin{theorem}
	\label{thm:descent_lemma_with_N_i_t}
	Let \allforone~be run on client $i$ in $\OneToN$ for $T$ in $\N$ iterations and collaborating as in \Cref{def:weights_binary}. Under \Cref{asu:bounded_variance,asu:smoothness}, taking a \emph{constant} step-size $\eta_i  <  (\sigma_{i, \phi}^t/ \sigma_{i, \psi}^t)^2 \smoothness^{-1}$, we have:
	\begin{align*}
		\FullExpec{\varepsilon_{i}^T} \leq \bigpar{1 - \eta_i \mu  }^T \varepsilon_{i}^0 + \frac{\beta \eta_i^2}{2} \sum_{t=1}^T (1-\eta_i \mu)^{T-t}(\sigma_{i, \psi}^t)^2 \,.
	\end{align*}
\end{theorem}

As shown in \Cref{thm:descent_lemma_with_N_i_t}, the \emph{effective variance} governing the convergence is given by $\sigma^t_{i, \mathrm{eff}} \coloneqq \sigma^t_{i, \psi}$. For clarity, we adopt the notation $\sigma^t_{i, \mathrm{eff}}$ in place of $\sigma^t_{i, \psi}$ when referring to the effective variance, as we believe this improves readability.
\Cref{thm:descent_lemma_with_N_i_t} highlights two terms: the exponential decrease of the initial bias $\varepsilon_i^0$, and the effective variance induced by the clients collaborating through the training. This rate of convergence precisely takes into account the impact of each client at each iteration. 

Because the effective variance $(\sigma_{i, \mathrm{eff}}^t)^2$ depends on $(r_{ik}^t)_{k\in \OneToN}$, it changes at each iteration, and therefore we cannot quantify precisely the speed-up induced by collaboration. To this aim, and using \Cref{asu:bounded_heterogeneity}, we introduce the notion of \textit{sufficient cluster} and \emph{sufficient variance}. Those are fixed in time, allowing to derive a looser but simpler upper-bound, and exhibit a sufficient set of clients whose collaboration allows reaching an optimization error smaller than $\varepsilon$, referred to as \emph{precision}, with a quantifiable speed-up.  

\begin{definition}[Sufficient cluster and variance at precision $\varepsilon$]
	\label{def:effectif_cluster_var}
	For a client $i$ in $\OneToN$, for a given final precision $\varepsilon$ in $\R_+$, we define the \emph{sufficient cluster} of clients collaborating with client $i$ to reach a precision $\varepsilon$ as $\mathcal{N}_i^\star(\varepsilon) \coloneqq \{k \in \OneToN \mid \psi((1 - \frac{b_{ik}}{2 \mu \varepsilon} - c_{ik})_+) > 0\}$. And the \emph{sufficient variance} is $\sigma_{i,{\mathrm{suf}}}^2(\varepsilon) \coloneqq (\sum_{k\in \OneToN} \sigma_k^{-2}\psi((1 - \frac{b_{ik}}{2 \mu \varepsilon} - c_{ik})_+))^{-1}$.
\end{definition}

\begin{wrapfigure}[11]{R}{0.5\linewidth}
	\centering
	\vspace{-0.5cm}
	\begin{subfigure}{\linewidth}
		\includegraphics[width=\linewidth]{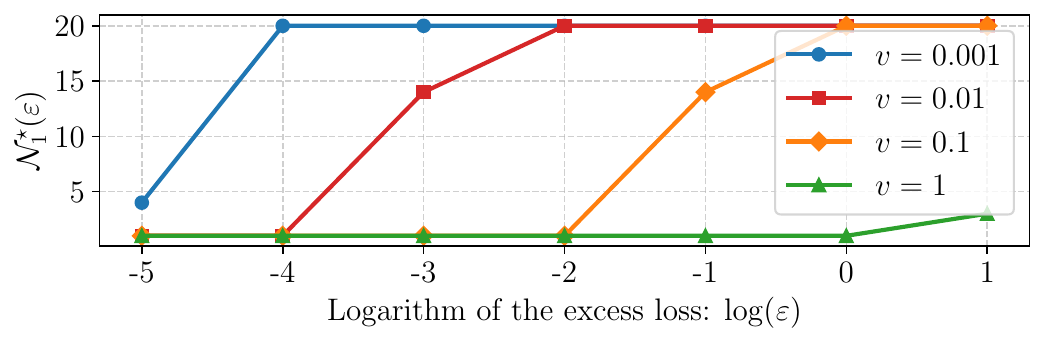}
	\end{subfigure}
	\caption{\label{fig:sufficient_cluster} Size of the sufficient cluster $\mathcal{N}_1^\star(\varepsilon)$  using synthetic data ($N=20$) and the binary collaboration function $\phi_{\mathrm{bin}}$ defined in \Cref{sec:application}. The variance $v$ of models $(\theta^\star_i)_{i\in\OneToN}$ is defined in \Cref{app:sec:experiments}: lower $v$ means lower heterogeneity.}
\end{wrapfigure}

\begin{remark}
	The sufficient cluster is purely theoretical and does not corresponds in practice to the clients selected by our schema from \cref{def:weights_binary}.
\end{remark}

Notice that the function $\sigma_{i,{\mathrm{suf}}}^2$ is decreasing and that it admits a limit as $\varepsilon$ decreases. Indeed, $\lim_{\varepsilon \to 0} \sigma_{i,{\mathrm{suf}}}^2(\varepsilon) = (\sum_{k\in \OneToN} \sigma_k^2 \psi((1- c_{ik})_+) \mathbb{1}_{b_{ik} = 0})^{-1}$ -- it corresponds to the case when there is collaboration only with client s.t. $b_{ik} = 0$ and $c_{ik} < 1$ -- and $\lim_{\varepsilon \to \infty} \sigma_{i,{\mathrm{cont}}}^2(\varepsilon) = (\sum_{k=1}^N \frac{1}{\sigma_k^2} \psi((1 - c_{ik})_+))^{-1}$. Observe that if $c_{ik} \geq 1$, a client will be excluded to the sufficient cluster, while it may be included in the effective cluster and help to reduce the variance. This is why, it corresponds to a cluster giving a sufficient but not necessary condition to reach a precision.

To provide intuition on the sufficient cluster, \Cref{fig:sufficient_cluster} illustrates how the sufficient cluster size evolves with the target excess loss $\varepsilon$, based on synthetic data (\Cref{app:sec:experiments}). Two observations emerge:
\begin{itemize}[leftmargin=*, itemsep=0pt, parsep=0pt, partopsep=0pt]
	\item The sufficient cluster size varies with the desired level of precision: while more clients participate at a lower precision level, achieving higher precision (with smaller $\varepsilon$) however  necessitate a change in the cluster composition and to collaborate with less clients.  
	\item The sufficient cluster size varies with the heterogeneity level. Under high heterogeneity (green line, $v = 1$), collaboration immediately vanishes as the excess loss $\varepsilon$ decreases, i.e. at the very beginning of the optimization. On the other side, with low heterogeneity (blue line, $v=0.001$), global collaboration remains beneficial throughout the optimization process, except for very small $\varepsilon$. In the intermediate regime (red and orange line), collaboration evolves over time: at low precision (high $\varepsilon$), it is global, the progressively becomes limited to similar clients before fully transitioning to individual fine-tuning. 
\end{itemize}

This concept allow us to derive, from \Cref{thm:descent_lemma_with_N_i_t}, a new upper bound on the convergence rate which directly depends on the sufficient variance as shown in \Cref{lem:nested_set_of_collaborating_clients}.

\begin{lemma}
	\label{lem:nested_set_of_collaborating_clients}
	Consider strong-convexity (\Cref{asu:strongly_convex}) or PL-condition (\Cref{asu:PL}), for any client $i$ in $\OneToN$, for any $t$ in $\N$, for any $\varepsilon$ in $\R_+^\star$ s.t. $\varepsilon \leq \varepsilon_{i}^t$, then  we have $\mathcal{N}_i^\star(\varepsilon) \subset \mathcal{N}_i^\star(\varepsilon_{i}^t) \subset \mathcal{N}_i^t$ and $ \sigma_{i,{\mathrm{eff}}}^t \leq \sigma_{i,{\mathrm{suf}}}(\varepsilon_i^t) \leq \sigma_{i,{\mathrm{suf}}}(\varepsilon)$.
\end{lemma}

\begin{proof}
	Let a client $i$ in $\OneToN$, $t$ in $\N$, any $\varepsilon$ in $\R_+^\star$ s.t. $\varepsilon \leq \varepsilon_{i}^t$.
	Then $\sqrdnrm{\nabla R_{i}(\theta^t_i) - \nabla R_{k}(\theta^t_i)} \leq b_{ik}^2 + c_{ik} \sqrdnrm{\nabla R_{i}(\theta^t_k)}$ (\Cref{asu:bounded_heterogeneity}) and by strong-convexity (\Cref{asu:strongly_convex}) or PL-condition (\Cref{asu:PL}), $\sqrdnrm{\nabla R_{i}(\theta^t_i)} \geq 2 \mu \varepsilon_{i}^t \geq 2 \mu \varepsilon$.
	Therefore, we have:
	$$ 1 -\frac{\sqrdnrm{\nabla R_{i}(\theta^t_i) - \nabla R_{k}(\theta^t_i)}}{\sqrdnrm{\nabla R_{i}(\theta^t_i)}} \geq 1 - \frac{b_{ik}^2 + c_{ik} \sqrdnrm{\nabla R_{i}(\theta^t_k)}}{\sqrdnrm{\nabla R_{i}(\theta^t_i)}} \geq 1 - \frac{b_{ik}^2}{2 \mu \varepsilon_{i}^t} - c_{ik} \geq 1 -\frac{b_{ik}^2}{2 \mu \varepsilon} - c_{ik}\,,$$
	which means that every element in $\mathcal{N}_i^\star(\varepsilon)$ is also in $\mathcal{N}_i^t$, however the converse is not necessarily true. And the result on the variance is obtained from the ratio inequality because $\psi$ is a non-negative and non-decreasing function.
\end{proof}

Building on \Cref{thm:descent_lemma_with_N_i_t,def:effectif_cluster_var,lem:nested_set_of_collaborating_clients}, we present three convergence rate for different step-sizes, identifying the linear and sub-linear regimes of convergence. 

\begin{theorem}
	\label{thm:convergence}
	Consider \Cref{asu:bounded_variance,asu:smoothness,asu:bounded_heterogeneity}, and the loss function to be either strongly-convex (\Cref{asu:strongly_convex}) or with the PL-condition in the non-convex setting (\Cref{asu:PL}). Let any precision level $\varepsilon>0$, then after running \allforone~for   $T$ iterations in $\N^\star$ s.t. $\varepsilon_i^{T-1}\geq \varepsilon$, we have the following upper-bounds on $\fullexpec{\varepsilon_{i}^T}$ for constant, horizon-dependant or decreasing step-sizes s.t. $\eta_{ik}^t  \leq (\sigma_{i, \phi}^t/ \sigma_{i, \psi}^t)^2 \beta^{-1}$.
	
	\begin{table}[H]
		\centering
		\caption{Upper-bounds on the excess loss $\fullexpec{\varepsilon_{i}^T}$ for three different step-size.}
		\begin{tabular}{l|l}
			\toprule
			Step size at iteration $t$ in $\N$ & Upper-bound\\
			\midrule
			Constant $\eta_i^t \leq \mu^{-1}$ & $\fullexpec{\varepsilon_{i}^T} \leq\bigpar{1 - \eta \mu }^T \varepsilon_{i}^0 + \ffrac{\eta \smoothness}{2 \mu} \sigma_{i,{\mathrm{suf}}}^2(\varepsilon)$ \\
			Horizon-dep. $\eta_i^t = \frac{1}{\mu T} \ln \bigpar{\frac{2 T \mu^2 \varepsilon_i^0}{\smoothness \sigma_{i,{\mathrm{suf}}}^2(\varepsilon)}} \leq \mu^{-1}$ &  $\fullexpec{\varepsilon_{i}^T} \leq \ffrac{\smoothness \sigma_{i,{\mathrm{suf}}}^2(\varepsilon)}{2 \mu^2 T} \bigpar{\ln\bigpar{\frac{2 T \mu^2 \varepsilon_{i}^0}{\sigma_{i,{\mathrm{suf}}}^2(\varepsilon)}} + 1}$ \\
			Decreasing $\eta_i^t = \frac{C}{\mu t}$ with $C > 1$ & $\fullexpec{\varepsilon_{i}^T} \leq\ffrac{1}{T} \max\bigpar{\ffrac{\smoothness \sigma_{i, \mathrm{suf}}^2(\varepsilon) C^2}{2 \mu^2 ( C - 1)}, \varepsilon_{i}^0}$ \\
			\bottomrule
		\end{tabular}
		\label{tab:summary_res_convergence}
	\end{table}
\end{theorem}

\begin{remark}
	For any $t$ in $\OneToN$ and any client $i$ in $\OneToN$, the ratio $\sigma_{i, \phi}^t/ \sigma_{i, \psi}^t$ is upper and lower-bounded by some constants, ensuring the possibility to take a constant step-size. Indeed, notice that $\forall x \in [0,1]$ $\psi(x) \leq \phi(x)$, we have $\sigma_{i, \phi}^t \leq \sigma_{i, \psi}^t$. Moreover, for the binary case we have $(\sigma_{i, \phi}^t/ \sigma_{i, \psi}^t)^2 \geq \lambda$ as $\psi(x) \geq \lambda \phi(x)$; and for the continuous case $(\sigma_{i, \phi}^t/ \sigma_{i, \psi}^t)^2 \geq \frac{1}{\sigma_i^{2}} (\sum_{k\in \OneToN} \frac{1}{\sigma_k^{2}})^{-1}$ as $\frac{1}{(\sigma_{i, \psi}^t)^{2}} = (\frac{1}{\sigma_i^2} + \sum_{k\neq i} \frac{1}{\sigma_k^2} \psi(r_{ik})) \geq \frac{1}{\sigma_i^2}$ and $(\sigma_{i, \phi}^t)^2 = (\sum_{k \in \OneToN} \frac{1}{\sigma_k^2} \psi(r_{ik})^{-1} \geq (\sum_{k \in \OneToN} \frac{1}{\sigma_k^2})^{-1} $.
\end{remark}

The proof is placed in \Cref{app:subsec:proof_cvgce}. These results recover classical SGD behavior with the novelty that the impact of personalized collaboration is embedded in the variance term $\sigma_{i, \mathrm{suf}}^2(\varepsilon)$. Several observations and comparisons with previous works can be made:
\begin{itemize}[leftmargin=*, itemsep=0pt, parsep=0pt, partopsep=0pt]
	\item The variance speed-up reduce with the precision $\varepsilon$. Running \allforone~for more iterations decreases $\varepsilon^{T-1}$, leading to a smaller acceptable $\varepsilon$, and consequently to an increased $\sigma_{i, \mathrm{suf}}^2(\varepsilon)$. 
	\item With a \textit{constant step-size}, we obtain a linear convergence rate up to a saturation level proportional to the sufficient variance. This matches classical SGD with constant step-size where variance prevent exacts convergence. With a \textit{decreasing} or \textit{horizon-dependent step-size}, we recover sublinear convergence without saturation.
	\item \citet{even2022sample} obtain a convergence up to a level depending on the heterogeneity forcing them to fix \emph{before the training} the value of the weight. Having dynamic weights allows to remove this dependence. \citet{chayti2021linear} also determine the weight before the training. They obtain (1) an exponential decrease of the bias, (2) a term equivalent to our sufficient variance, and (3) an additional term of bias due to heterogeneity which is removed by introducing a bias correction mechanism. This is why, our algorithm can also be seen as a \emph{bias reduction} method.
	\item Compared to \citet[][Theorem I, Corollary II]{hashemi2024cobo}, who assume the existence of latent clusters of size $c$ and exhibits a $1/c$ speed-up only on the model drift, our formulation yields a $1/\mathcal{N}_i^\star(\varepsilon)$ improvement also on the \emph{excess loss} (term hidden in the sufficient variance). Furthermore, while they focus solely on horizon-dependent step-sizes, we use both constant and decreasing step-size regimes, explicitly highlighting both linear and sublinear convergence. 
\end{itemize}

In the following corollary, we exhibit a bound on the number of iteration necessary to reach a precision $\varepsilon>0$. This type of result is naturally obtained from our notion of sufficient variance, which provides an analysis of the excess loss for precision above $\varepsilon$.

\begin{corollary}[Sampling complexity]
	\label{cor:sampling_complexity_optimality}
	Let $\varepsilon > 0$ and a client $i$ in $\OneToN$, taking a decreasing step-size as in \Cref{thm:convergence}, \allforone~reach a generalisation error $\varepsilon$ after $T_i^\varepsilon = \beta \sigma_{i, \mathrm{suf}}^2(\varepsilon) \mathcal{C} / (2 \mu^2 \varepsilon)$ iterations, where $\mathcal{C} = C^2 / (C - 1)$. Under the additional assumptions that all clients share the same gradient variance $\sigma^2$ and that $c_{ik} = 0$ for all pairs $i,k \in \OneToN$, \Cref{thm:convergence} gives:
	$$T_i^\varepsilon \leq \ffrac{\mathcal{C} \beta \sigma^2}{2 \mu^2 \varepsilon \mathcal{N}_i^\star(\varepsilon) \min_{j \in \mathcal{N}_i^\star(\varepsilon)} \psi(1 - \frac{b_{ik}}{2 \mu \varepsilon}) } \,.
	$$  
\end{corollary}

In the case of non-convex objective function, \Cref{asu:PL} is true only if the model is in the neighborhood of a local optimum points, but in general it is not the case. We therefore propose a bound of convergence in the scenario, recovering the classical $O(T^{-1/2})$ rate. 

\begin{theorem}
	\label{thm:convergence_non_cvx}
	Consider \Cref{asu:bounded_variance,asu:smoothness,asu:bounded_heterogeneity}  in the non-convex setting. Let any precision level $\varepsilon>0$, then after running \allforone~for   $T$ iterations in $\N^\star$ s.t. $\varepsilon_i^{T-1}\geq \varepsilon$, taking a constant horizon-dependent step-size $\eta_i^t = \eta_i = \min(\sqrt{2 \varepsilon_{i}^0} (T \smoothness \sigma_{i,{\mathrm{suf}}}^2(\varepsilon))^{-1/2}, (\sigma_{i, \phi}^t/ \sigma_{i, \psi}^t)^2 \beta^{-1})$ we have: 
	\begin{align*}
		\frac{1}{T} \sum_{t=1}^{T} \FullExpec{\sqrdnrm{\nabla R_{i}(\theta^{t-1}_i)}} \leq 2 \sqrt{\frac{2 \varepsilon_{i}^0 \smoothness \sigma_{i,{\mathrm{suf}}}^2(\varepsilon)}{T}}\,.
	\end{align*}  
\end{theorem}

\section{Application with two collaboration criterions}
\label{sec:application}

In this section, we explain how to dynamically estimate in practice the weights $(\alpha^t_{ik})_{i \in \OneToN, k \in \OneToK}$  and propose two methods to instantiate the criterion function $\phi$.

The first proposed schema is a \textit{gradient filter} method based on ``binary'' weights 
(hard-thresholding). This is as opposed to the ``continuous'' collaboration scheme, where $\phi$ is continuous at $0$, i.e. two clients can collaborate while having weights arbitrarily close to $0$ (soft-thresholding).

\begin{definition}[Binary and continuous collaborations]
	For the binary collaboration, we take $\phi_{\mathrm{bin}} : x \mapsto \lambda \mathbb{1}\{ x \geq \lambda \}$. The indicator function is multiplied by a hyperparameter $\lambda$ in order to have for any $x$ in $[0, 1]$, $\phi_{\mathrm{bin}}(x) \leq x$ (required by \Cref{def:weights_binary}). For the continuous collaboration we take $\phi_{\mathrm{cont}} : x \mapsto x$.
\end{definition}

If we assume that all the $(c_{ik})_{k \in \OneToN}$ are equal to zero and that the variance is uniformly bounded across clients, i.e. for all $k$ in $\OneToN$, $\sigma_k = \sigma$, then we can prove that $\sigma_{i, \mathrm{suf}}^2 > \sigma^2 (\psi(1)\mathcal{N}_i^\star(\varepsilon))^{-1}$. Indeed, for both the binary and the continuous variant, for any $x$ in $[0, 1]$, we have $\psi(x) \leq \psi(1)$, it yields $ \sum_{k \in \OneToN} \psi (1- \frac{b_{ik}}{2 \mu \varepsilon}) = \sum_{k \in \mathcal{N}_i^\star(\varepsilon)} \psi (1- \frac{b_{ik}}{2 \mu \varepsilon}) \leq  \psi(1) \mathcal{N}_i^\star(\varepsilon)$ and therefore $\sigma_{i, \mathrm{suf}}^2(\varepsilon) = \sigma^2 (\sum_{k \in \mathcal{N}^\star_i} \psi(1 - \frac{b_{ik}}{2 \mu \varepsilon}))^{-1} \geq \sigma^2 ( \psi(1) \mathcal{N}^\star_i(\varepsilon))^{-1}$. All together, \Cref{cor:sampling_complexity_optimality} recovers the lower-bound $\Omega\big(\ffrac{\beta \sigma^2 }{\mu^2 \varepsilon \mathcal{N}_i^\star(\varepsilon)}\big)$ on $T_i^\varepsilon$ established by \citet{even2022sample}. Moreover, unlike their approach, which relies on a horizon-dependent step-size schedule, our results achieve this bound \emph{without} additional logarithmic factors. 
In addition, our results are more flexible than theirs. Indeed, since our scheme is independent of the final accuracy, it is optimal for any $\varepsilon > 0$. In contrast \citet{even2022sample} prove that for any $\varepsilon > 0$, there exists an algorithm restricted to the sufficient cluster that achieve optimality. 

For the binary variant as we have $\min_{j \in \mathcal{N}i^\star(\varepsilon)} \psi_{\mathrm{bin}}(1 - \frac{b{ik}}{2 \mu \varepsilon}) = \lambda$, the upper-bound on $T_i^\varepsilon$ match the lower-bound on the sample complexity $T_i^\varepsilon$.
However, for the continuous variant $\psi_{\mathrm{cont}}$, the term $\min_{j \in \mathcal{N}i^\star(\varepsilon)} \psi_{\mathrm{cont}}(1 - \frac{b{ik}}{2 \mu \varepsilon})$ can become arbitrarily small when local dissimilarities $b_{ik}$ grow large, potentially degrading the worse-case upper-bound. 

The reason for still considering the continuous version despite its apparent suboptimality is that it minimizes the upper bound in \Cref{lem:descent_lemma}. Indeed, optimizing over $\alpha_{ik}^t$ gives: 
\begin{align*}
	- \sqrdnrm{\nabla R_{i}(\theta^{t-1}_i)} + \sqrdnrm{\nabla R_{k, S}(\theta^{t-1}_i) - \nabla R_{i}(\theta^{t-1}_i)} +  2 \eta_i^t \smoothness \sigma^2_k \alpha_{ik}^{t} = 0 \,,
\end{align*}

and thus, it yields to $\alpha_{ik}^t$ proportional to $\frac{1}{\sigma_k^2} (\sqrdnrm{\nabla R_{i}(\theta^{t-1}_i)} - \sqrdnrm{\nabla R_{k, S}(\theta^{t-1}_i) - \nabla R_{i}(\theta^{t-1}_i)})_+$. We normalize this quantities by $\sqrdnrm{\nabla R_i(\theta_i^{t-1})}$ (independent of $k$, hence preserving the minimization of \Cref{lem:descent_lemma}) to make appear the similarity ratio $r_{ik}$. Moreover, this variant allows to remove the additional hyperparameter $\lambda$, making it more practical.

\paragraph{Estimating the collaboration weights.} 

Strictly speaking, estimating $\alpha_{ik}^t$ from \Cref{def:weights_binary} is not possible without knowing the true gradients. We propose a stochastic estimation by computing for any iteration $t$ in $\N^\star$ and any client $i$ in $\OneToN$, two random variables $Z_{ik}^t$ and $Z_{i}^t$ defined as follows:
\begin{align}
	\label{eq:estimating_alpha}
	Z_{ik}^t = \Big\|\frac{1}{m} \sum_{j=1}^m \left(g^j_i(\theta_i) - g^j_k(\theta_i)\right)\Big\|^2 \qquad \text{and} \qquad Z_i^t = \Big\|\frac{1}{m}\sum_{j=1}^m g^j_i(\theta_i)\Big\|^2\,,
\end{align} 
where $(g^j_i(\theta_i), g^j_k(\theta_i))_{j\in[b_{\alpha}]}$ are $b_{\alpha}$ in $\N^\star$ stochastic gradients independent of $g_i(\theta_i), g_k(\theta_i)$ computed on $b_{\alpha}$ different batches. Then we replace $r_{ik}^t$ by $\hat{r}_{ik}^{t} = ( 1 - Z_{ik}^t / Z_i)_ +$. We experimentally tested several approaches and this one gave the best results on our datasets. Furthermore, to lower computational overhead, one may also consider updating the weights less frequently -- e.g., per epoch or on a logarithmic schedule -- as their evolution is expected to be gradual. Although one might argue that allocating the additional $b_{\alpha}$ gradients to reduce local variance could be preferable, when $N \gg b_{\alpha}$ leveraging inter-client collaboration yields better variance reduction. While improving the proposed above estimator along with strong theoretical guarantee is a promising direction -- e.g., by drawing inspiration from adaptive methods such as Adam \citep{kingma2014adam} –- it is left for future work.

\section{Experiments}
\label{sec:experiments}

\begin{wrapfigure}[11]{L}{0.55\linewidth}
	\centering
	\vspace{-0.5cm}
	\begin{subfigure}{0.48\linewidth}
		\includegraphics[width=\linewidth]{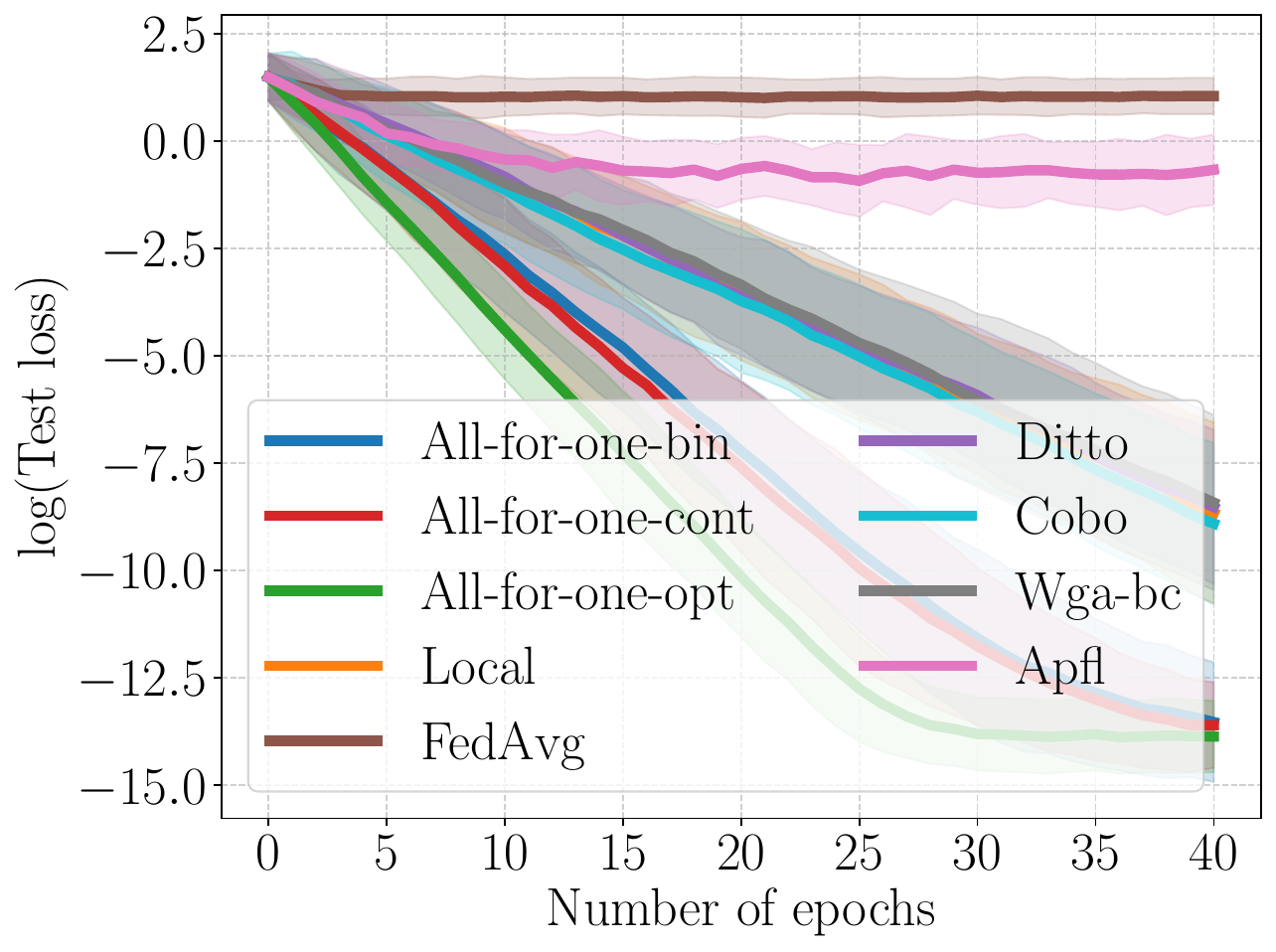}
		\caption{\label{fig:synth_d2} Synth. data with $d=2$.}
	\end{subfigure}
	\hfill
	\begin{subfigure}{0.48\linewidth}
		\includegraphics[width=\linewidth]{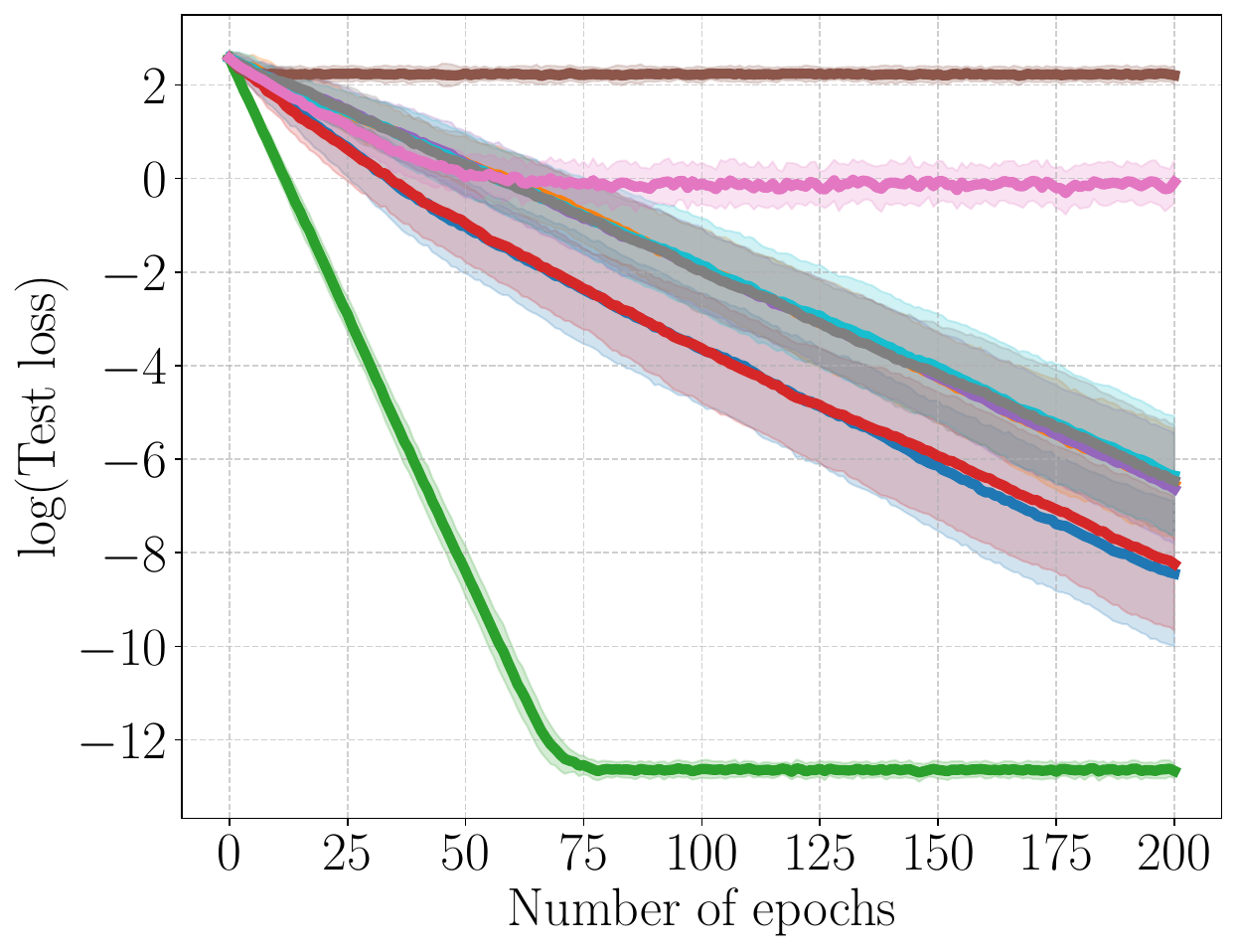}  
		\caption{\label{fig:synth_d10} Synth. data with $d=10$.}
	\end{subfigure}
	\caption{\label{fig:synth_data} Test loss for synthetic dataset.}
\end{wrapfigure}

We now evaluate the empirical performance of the collaboration scheme from~\Cref{sec:application}. Our code is provided on \href{https://github.com/philipco/adaptative_collaboration}{Github}. Complete experiments' settings are detailed in \Cref{app:sec:experiments}. All experiments are performed without any tuning of the algorithms (e.g. same set of hyperparameters for every algorithms), for three different seeds, without any pretraining. We use five datasets: one synthetic that allows to illustrate our theorems in a simple framework with two clusters (described in \Cref{app:sec:experiments}), two real with synthetic splits (Mnist with a CNN and Cifar10 with LeNet), and two real FL datasets (Heart Disease with a linear model and Ixi with Unet) put forward by \citet{terrail2022flamby}. To split Mnist and Cifar10 accross the $N=20$ client and to model statistical heterogeneity, we adopt the cluster partitioning scheme introduced in \Cref{app:def:cluster_split}.  The details setting of training for each dataset is given in \Cref{app:tab:settings_xp}.

We compare \allforone~(using the binary and continuous form) to \FedAvg, to \texttt{Local} training without any collaboration. For synthetic datasets/splits with known ground-truth clusters, we compare against the optimal \allforone~algorithm, in which each client collaborates exclusively with clients from the same cluster using uniform weights. This is the closest algorithm to the one proposed by \citep{even2022sample}, which weights are not computable in practice.  In addition, we compare our algorithm to the following related algorithms \texttt{Ditto} \citep{li2021ditto}, \texttt{Cobo} \citep{hashemi2024cobo}, \texttt{Wga-bc} \citep{chayti2021linear}, and \texttt{Apfl} \citep{deng2020adaptive}.
Note that to speed-up the experiments and reduce the computational cost, we update the weights after a few batch of iterations and not after each one (see \Cref{app:tab:settings_xp}). On \Cref{fig:synth_data,fig:real_dataset}, we plot the  average of test loss (in log scale) over the $N$ clients and the three seeds w.r.t. the number of epoch. The corresponding standard deviation is indicated with the shadow area. We report the final train and test loss/accuracy (averaged over clients/seeds and weighted by their relative dataset sizes) for real dataset on \Cref{tab:exp_real}, we highlight in green (resp. orange) for each dataset the algorithm that reach the best value on test (resp. train) set.

\textbf{Observations.}
\begin{itemize}[leftmargin=*, itemsep=0pt, parsep=0pt, partopsep=0pt]
	\item On synthetic data (\Cref{fig:synth_data}), due to the existence of two distinct clusters, \FedAvg~exhibits immediate saturation caused by global aggregation across heterogeneous clients. In contrast, all other methods display a sub-linear convergence arising from the fact that, in our least-squares regression setting, noise comes only from the sampling of the features and not from additive perturbations to the outputs. This setup is deliberately chosen to highlight this sublinear convergence. While \texttt{Ditto}, \texttt{Cobo}, \texttt{Wga-bc} and \texttt{Apfl} have the same convergence rate than \texttt{Local} training, \allforone~achieves an accelerated convergence rate by effectively leveraging collaboration with clients from the same cluster and therefore acts as a variance-reduction method.
	\item On \Cref{fig:synth_d10}, we observe two distinct convergence regimes for \allforone: an initial linear phase followed by a sub-linear one. In the early phase, \allforone~collaborates with all client from the same cluster enjoying a faster convergence, then dynamically shifts to a local training. This sustains progress, converging at a rate matching \texttt{Local} while avoiding \FedAvg's saturation and still retaining early-phase's advantage over \texttt{Local}.
	\item Although the optimal \allforone~algorithm presents the best convergence across the four scenarios, it is unimplementable in practice as it requires to know the true underlying clusters \emph{before} the training. It should be considered as a theoretical lower bound for the true \allforone, highlighting opportunities for future advancements in the evaluation of the weights $(\alpha_{ik})_{i,k \in \OneToN}$. Observe that on Cifar10, the true \allforone~almost reach the same performance, showcasing that it effectively identifies the best collaboration.
	\item \FedAvg~consistently underperforms across all datasets, yet having close train and test performances. 
	\item While \texttt{Local} training achieves high training accuracy, it often fails to generalize, with notable gaps between train and test performance. The exception is the Ixi dataset, where train and test scores are nearly identical. This suggests that the data distribution in Ixi is highly heterogeneous but also internally consistent within each client, making local overfitting less detrimental.
	\item On Cifar10 the binary \allforone~achieves the best generalization performance ($63.9\%$), significantly outperforming both \texttt{Local} (58.5\%) and \FedAvg~(28.9\%) in test accuracy. \texttt{Local}, \texttt{Ditto}, \texttt{Cobo} and \texttt{Apfl} overfit after $80$ epochs. The absolute performance remains below commonly reported benchmarks due to the use of a compact LeNet model, chosen for computational feasibility. This highlights that even under limited capacity, adaptive decentralized collaboration can yield substantial gains.
\end{itemize}
Overall, the $\allforone$ presents the best test performance with the smaller gap with the train performance, illustrating both its efficiency and its stability.

\begin{figure}
	\centering
	\centering     
	\begin{subfigure}{0.24\linewidth}
		\includegraphics[width=\linewidth]{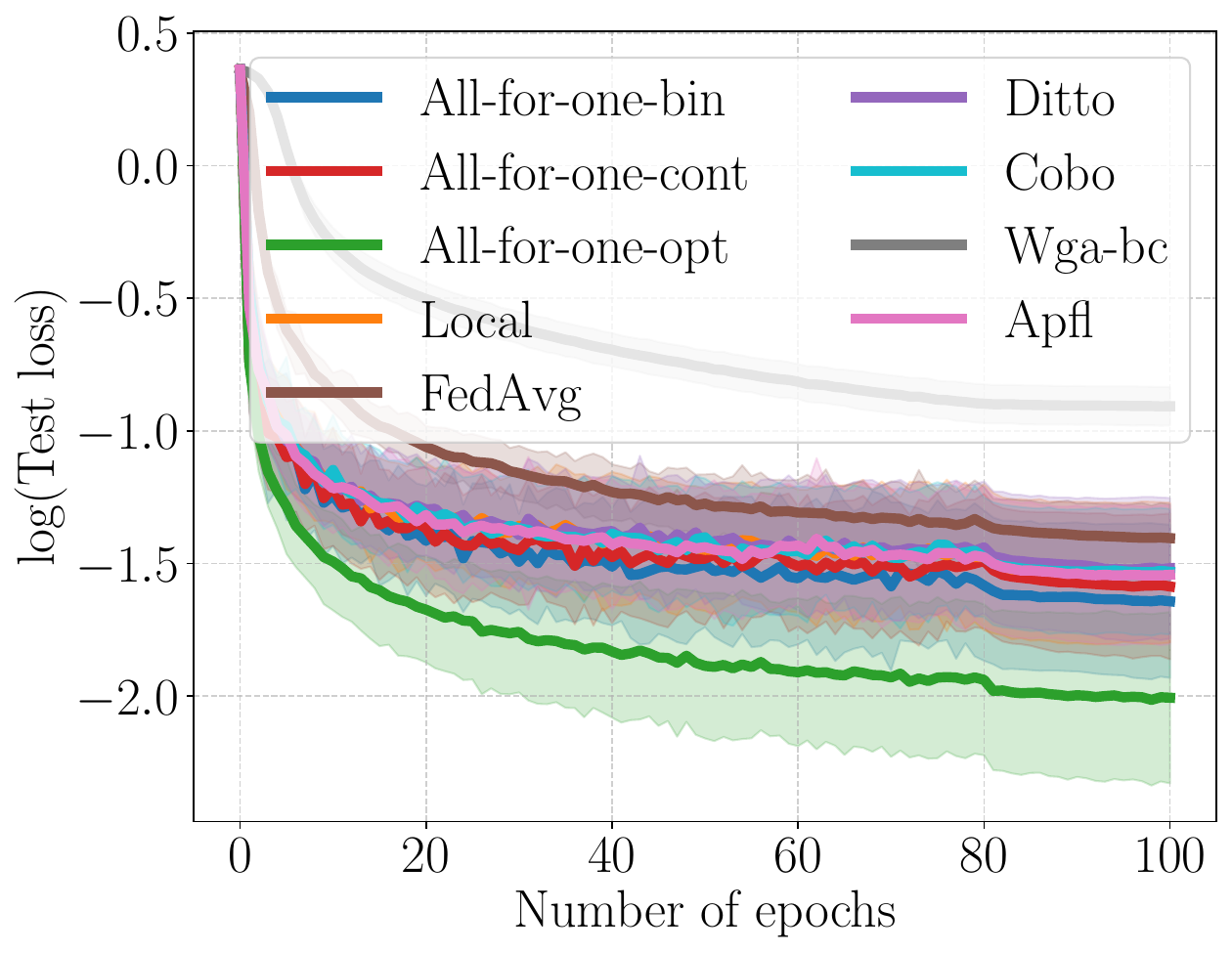}
		\caption{\label{fig:mnist} Mnist.} 
	\end{subfigure}
	\begin{subfigure}{0.24\linewidth}
		\includegraphics[width=\linewidth]{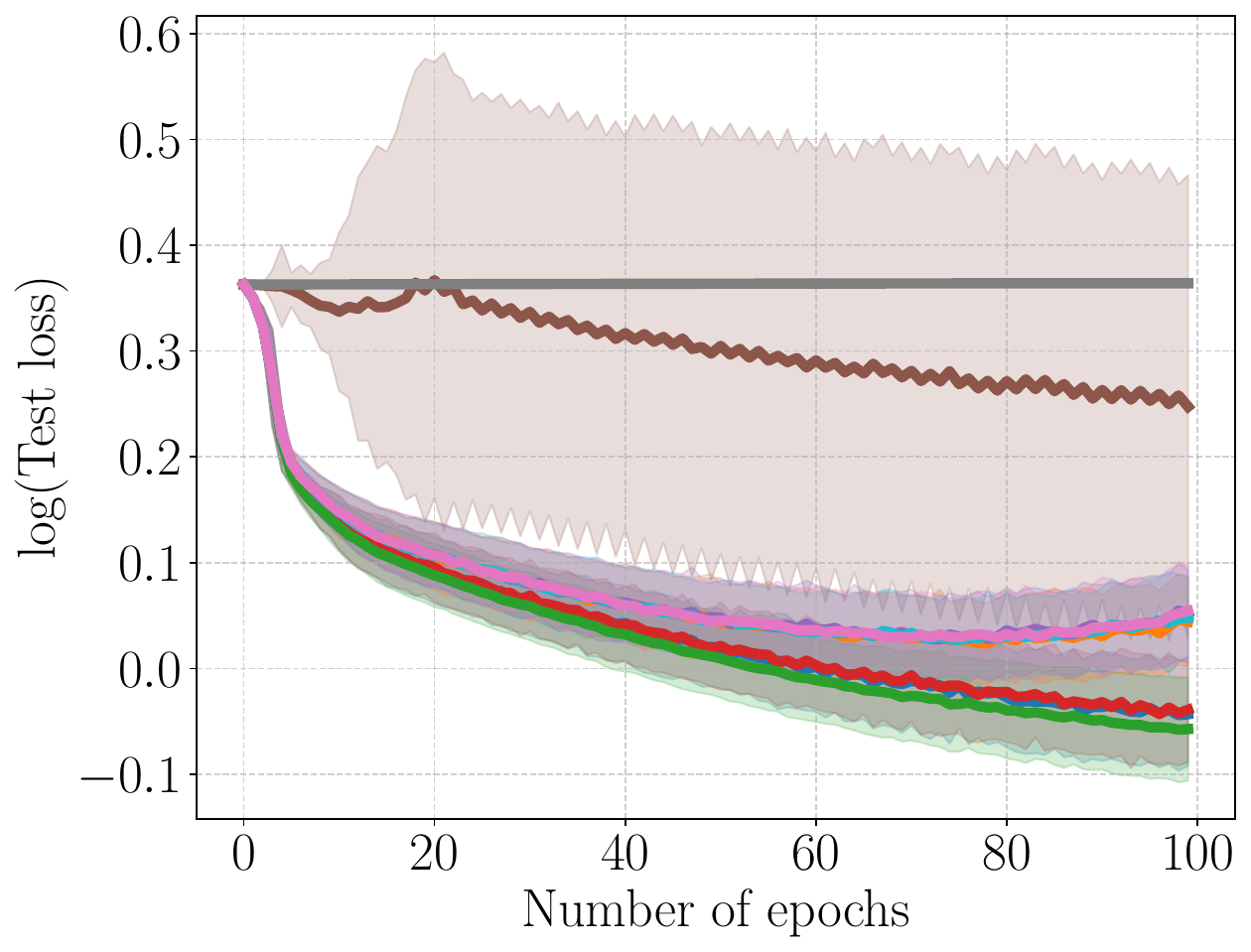}
		\caption{\label{fig:cifar10} Cifar10.}
	\end{subfigure}
	\begin{subfigure}{0.24\linewidth}
		\includegraphics[width=\linewidth]{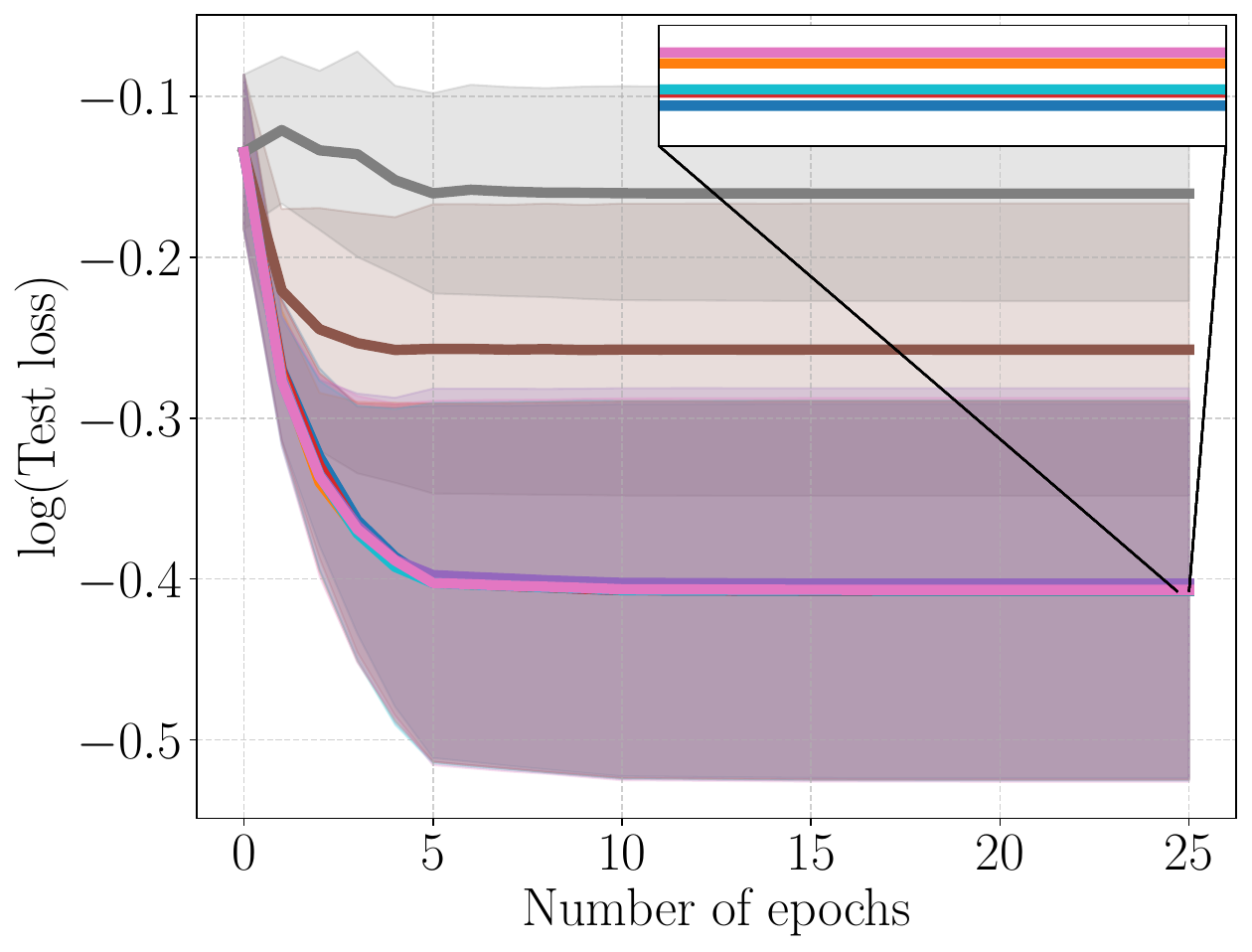}
		\caption{ \label{fig:heart_disease} Heart disease}
	\end{subfigure}
	\begin{subfigure}{0.24\linewidth}
		\includegraphics[width=\linewidth]{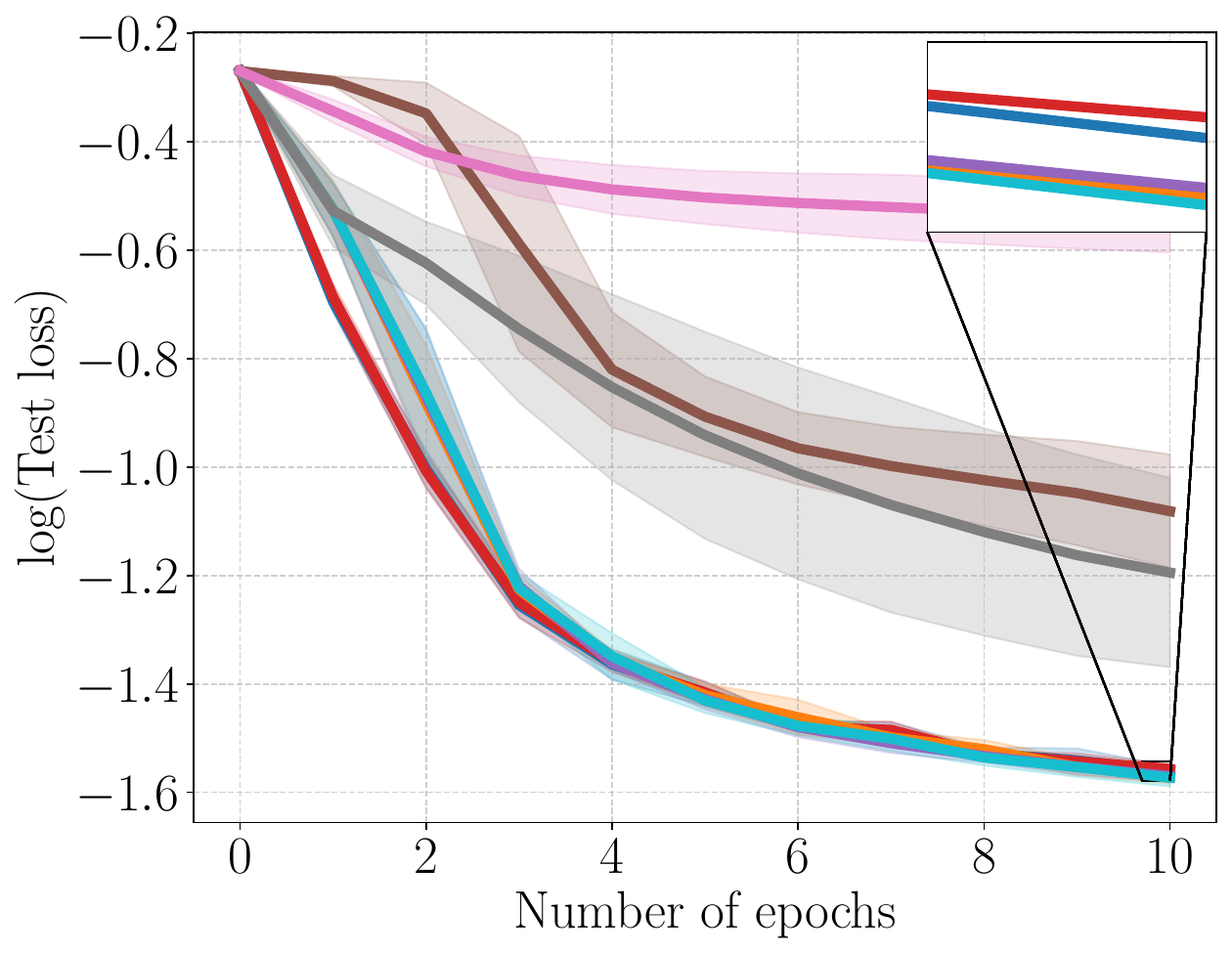}
		\caption{ \label{fig:ixi} Ixi}
	\end{subfigure}
	\caption{\label{fig:real_dataset} Test loss for real dataset. X-axis: number of epochs. Y-axis: logarithm of the test loss.}
\end{figure}


\begin{table}
	\centering
	\caption{Train and test accuracy/loss for the real datasets.}
	\label{tab:exp_real}
	\resizebox{\textwidth}{!}{%
		\begin{tabular}{p{0.09\textwidth}p{0.17\textwidth}cccccccc}
			\toprule
			& Algorithm & \multicolumn{2}{c}{Mnist} & \multicolumn{2}{c}{Cifar10} & \multicolumn{2}{c}{ Heart Disease} & \multicolumn{2}{c}{Ixi} \\
			& & Train & Test & Train & Test & Train & Test & Train & Test \\
			\midrule
			Accuracy & All-for-one-bin & \cellcolor{orange!25}100.0 & \cellcolor{green!25}99.2 & 69.9 & \cellcolor{green!25} 64.2 & \cellcolor{orange!25}88.9 & \cellcolor{green!25}82.3 & 97.3 & 97.2 \\ 
			(in $\%$) & All-for-one-cont & \cellcolor{orange!25}100.0 & 99.1 & 70.7 & 63.5 & \cellcolor{orange!25}88.9 & 82.1 & 97.3 & 97.2 \\ 
			&  Local & \cellcolor{orange!25}100.0 & 99.1 & 75.2 & 59.0 & 77.8 & 82.1 & 97.3 & \cellcolor{green!25}97.3 \\ 
			&  FedAvg & 98.9 & 98.7 & 30.1 & 30.8 & 62.2 & 75.2 & 91.5 & 91.5 \\ 
			&  Ditto & \cellcolor{orange!25}100.0 & 99.1 & \cellcolor{orange!25}76.5 & 58.5 & 77.8 & 81.9 & 97.3 & \cellcolor{green!25}97.3 \\ 
			&  Cobo & \cellcolor{orange!25}100.0 & 99.1 & 76.2 & 58.7 & \cellcolor{orange!25}88.9 & 81.9 & \cellcolor{orange!25}97.4 & \cellcolor{green!25}97.3 \\ 
			&  Wga-bc & 96.3 & 96.3 & 4.8 & 4.7 & 55.6 & 66.3 & 93.2 & 93.1 \\ 
			&  Apfl & \cellcolor{orange!25}100.0 & 99.1 & 75.8 & 58.7 & 77.8 & 82.1 & 70.4 & 70.4 \\ 
			Loss & All-for-one-bin & -4.01 & \cellcolor{green!25} -1.64 & -0.11 & \cellcolor{green!25}-0.04 & -0.67 & \cellcolor{green!25}-0.41 & -1.57 & -1.56 \\ 
			(in log) & All-for-one-cont & -3.99 & -1.59 & -0.12 & \cellcolor{green!25}-0.04 & \cellcolor{orange!25}-0.68 & \cellcolor{green!25}-0.41 & -1.57 & -1.56 \\ 
			&  Local & \cellcolor{orange!25}-4.10 & -1.53 & -0.21 & 0.05 & -0.52 & \cellcolor{green!25}-0.41 & \cellcolor{orange!25}-1.58 & \cellcolor{green!25}-1.57 \\ 
			&  FedAvg & -1.61 & -1.40 & 0.26 & 0.26 & -0.22 & -0.26 & -1.08 & -1.08 \\ 
			&  Ditto & -3.86 & -1.52 & \cellcolor{orange!25}-0.22 & 0.06 & -0.41 & -0.40 & \cellcolor{orange!25}-1.58 & \cellcolor{green!25}-1.57 \\ 
			&  Cobo & -4.01 & -1.53 & -0.21 & 0.05 & -0.60 & \cellcolor{green!25}-0.41 & \cellcolor{orange!25}-1.58 & \cellcolor{green!25}-1.57 \\ 
			&  Wga-bc & -0.94 & -0.91 & 0.36 & 0.36 & -0.12 & -0.16 & -1.20 & -1.19 \\ 
			&  Apfl & -3.94 & -1.54 & \cellcolor{orange!25}-0.22 & 0.05 & -0.43 & \cellcolor{green!25}-0.41 & -0.53 & -0.54 \\ 
			\bottomrule
		\end{tabular}
	}
\end{table}

\section{Conclusion}
\label{sec:ccl}

In this work, we establish a relationship between clients statistical heterogeneity, gradients at each iteration and collaboration in distributed personalized learning. This results in a new accuracy-driven collaboration criterion that adapts dynamically without requiring assumption on the existence of underlying fixed cluster. We provide convergence guarantees over strongly-convex functions, non-convex functions, and functions satisfying the PL condition, and recover the optimal sample complexity.

\textbf{Broader impact.}
This work open the door for better collaboration and specialization of model, which can be very beneficial for different tasks,  LLM for instance \citep{liu2023dynamic}.
Furthermore, by quantifying each client’s utility to the training, our framework opens new direction to incentive mechanisms and rewarding clients for their contribution \citep{karimireddy2022mechanisms}. Additionally, the adaptive selection of collaborators enhances robustness to adversarial or misaligned clients, contributing to fairness and reliability in decentralized learning systems \citep{yu2020salvaging,li2021ditto}.

\textbf{Limitations and open directions.}
First, while our theory establish the relationship between true gradients and collaboration, in practice, we rely on a stochastic oracle; improving the gradient estimator is a valuable extension. Second, we assume uniform bounds on gradient noise -- a standard but potentially restrictive assumption. Third, we do not account for partial device participation, i.e., scenarios where only a random subset of clients is active at each round.

\section{Acknowledgment}

This work was supported by the French government managed by the Agence Nationale de la Recherche (ANR) through France 2030 program with the reference ANR-23-PEIA-005 (REDEEM project). It was also funded in part by the Groupe La Poste, sponsor of the Inria Foundation, in the framework of the FedMalin Inria Challenge. Laurent Massoulié was supported by the French government under management of Agence Nationale de la Recherche as part of the “Investissements d’avenir” program, reference ANR19- P3IA-0001 (PRAIRIE 3IA Institute).

The authors are grateful to the CLEPS infrastructure from the Inria of Paris for providing resources and support.

\bibliographystyle{abbrvnat} 
\bibliography{main.bib}

\begin{thebibliography}{37}
\providecommand{\natexlab}[1]{#1}
\providecommand{\url}[1]{\texttt{#1}}
\expandafter\ifx\csname urlstyle\endcsname\relax
  \providecommand{\doi}[1]{doi: #1}\else
  \providecommand{\doi}{doi: \begingroup \urlstyle{rm}\Url}\fi

\bibitem[Arivazhagan et~al.(2019)Arivazhagan, Aggarwal, Singh, and
  Choudhary]{arivazhagan2019federated}
M.~G. Arivazhagan, V.~Aggarwal, A.~K. Singh, and S.~Choudhary.
\newblock Federated learning with personalization layers.
\newblock \emph{arXiv preprint arXiv:1912.00818}, 2019.

\bibitem[Beaussart et~al.(2021)Beaussart, Grimberg, Hartley, and
  Jaggi]{beaussart2021waffle}
M.~Beaussart, F.~Grimberg, M.-A. Hartley, and M.~Jaggi.
\newblock Waffle: Weighted averaging for personalized federated learning.
\newblock \emph{arXiv preprint arXiv:2110.06978}, 2021.

\bibitem[Bottou et~al.(2018)Bottou, Curtis, and
  Nocedal]{bottou2018optimization}
L.~Bottou, F.~E. Curtis, and J.~Nocedal.
\newblock Optimization methods for large-scale machine learning.
\newblock \emph{SIAM review}, 60\penalty0 (2):\penalty0 223--311, 2018.

\bibitem[Capitaine et~al.(2024)Capitaine, Boursier, Scheid, Moulines, Jordan,
  El-Mhamdi, and Durmus]{capitaine2024unravelling}
A.~Capitaine, E.~Boursier, A.~Scheid, E.~Moulines, M.~Jordan, E.-M. El-Mhamdi,
  and A.~Durmus.
\newblock Unravelling in collaborative learning.
\newblock \emph{Advances in Neural Information Processing Systems},
  37:\penalty0 97231--97260, 2024.

\bibitem[Chayti et~al.(2021)Chayti, Karimireddy, Stich, Flammarion, and
  Jaggi]{chayti2021linear}
E.~M. Chayti, S.~P. Karimireddy, S.~U. Stich, N.~Flammarion, and M.~Jaggi.
\newblock Linear speedup in personalized collaborative learning.
\newblock \emph{arXiv preprint arXiv:2111.05968}, 2021.

\bibitem[Collins et~al.(2021)Collins, Hassani, Mokhtari, and
  Shakkottai]{collins2021exploiting}
L.~Collins, H.~Hassani, A.~Mokhtari, and S.~Shakkottai.
\newblock Exploiting shared representations for personalized federated
  learning.
\newblock In M.~Meila and T.~Zhang, editors, \emph{Proceedings of the 38th
  International Conference on Machine Learning}, volume 139 of
  \emph{Proceedings of Machine Learning Research}, pages 2089--2099. PMLR,
  18--24 Jul 2021.
\newblock URL \url{https://proceedings.mlr.press/v139/collins21a.html}.

\bibitem[Deng et~al.(2020)Deng, Kamani, and Mahdavi]{deng2020adaptive}
Y.~Deng, M.~M. Kamani, and M.~Mahdavi.
\newblock Adaptive personalized federated learning.
\newblock \emph{arXiv preprint arXiv:2003.13461}, 2020.

\bibitem[Even et~al.(2022)Even, Massouli{\'e}, and Scaman]{even2022sample}
M.~Even, L.~Massouli{\'e}, and K.~Scaman.
\newblock On sample optimality in personalized collaborative and federated
  learning.
\newblock \emph{Advances in Neural Information Processing Systems},
  35:\penalty0 212--225, 2022.

\bibitem[Fallah et~al.(2020)Fallah, Mokhtari, and
  Ozdaglar]{fallah2020personalized}
A.~Fallah, A.~Mokhtari, and A.~Ozdaglar.
\newblock Personalized federated learning with theoretical guarantees: A
  model-agnostic meta-learning approach.
\newblock \emph{Advances in neural information processing systems},
  33:\penalty0 3557--3568, 2020.

\bibitem[Ghosh et~al.(2020)Ghosh, Chung, Yin, and
  Ramchandran]{ghosh2020efficient}
A.~Ghosh, J.~Chung, D.~Yin, and K.~Ramchandran.
\newblock An efficient framework for clustered federated learning.
\newblock \emph{Advances in neural information processing systems},
  33:\penalty0 19586--19597, 2020.

\bibitem[Hanzely et~al.(2020)Hanzely, Hanzely, Horv{\'a}th, and
  Richt{\'a}rik]{hanzely2020lower}
F.~Hanzely, S.~Hanzely, S.~Horv{\'a}th, and P.~Richt{\'a}rik.
\newblock Lower bounds and optimal algorithms for personalized federated
  learning.
\newblock \emph{Advances in Neural Information Processing Systems},
  33:\penalty0 2304--2315, 2020.

\bibitem[Hashemi et~al.(2024)Hashemi, He, and Jaggi]{hashemi2024cobo}
D.~Hashemi, L.~He, and M.~Jaggi.
\newblock Cobo: Collaborative learning via bilevel optimization.
\newblock \emph{arXiv preprint arXiv:2409.05539}, 2024.

\bibitem[Karimireddy et~al.(2020)Karimireddy, Kale, Mohri, Reddi, Stich, and
  Suresh]{karimireddy2020scaffold}
S.~P. Karimireddy, S.~Kale, M.~Mohri, S.~Reddi, S.~Stich, and A.~T. Suresh.
\newblock Scaffold: Stochastic controlled averaging for federated learning.
\newblock In \emph{International conference on machine learning}, pages
  5132--5143. PMLR, 2020.

\bibitem[Karimireddy et~al.(2022)Karimireddy, Guo, and
  Jordan]{karimireddy2022mechanisms}
S.~P. Karimireddy, W.~Guo, and M.~I. Jordan.
\newblock Mechanisms that incentivize data sharing in federated learning.
\newblock \emph{arXiv preprint arXiv:2207.04557}, 2022.

\bibitem[Kingma and Ba(2014)]{kingma2014adam}
D.~P. Kingma and J.~Ba.
\newblock Adam: A method for stochastic optimization.
\newblock \emph{arXiv preprint arXiv:1412.6980}, 2014.

\bibitem[Konečný et~al.(2016)Konečný, McMahan, Yu, Richtarik, Suresh, and
  Bacon]{konecny_federated_2016}
J.~Konečný, H.~B. McMahan, F.~X. Yu, P.~Richtarik, A.~T. Suresh, and
  D.~Bacon.
\newblock Federated {Learning}: {Strategies} for {Improving} {Communication}
  {Efficiency}.
\newblock In \emph{{NIPS} {Workshop} on {Private} {Multi}-{Party} {Machine}
  {Learning}}, 2016.

\bibitem[Krizhevsky et~al.(2009)Krizhevsky, Hinton,
  et~al.]{krizhevsky_learning_2009}
A.~Krizhevsky, G.~Hinton, et~al.
\newblock Learning multiple layers of features from tiny images.
\newblock 2009.

\bibitem[Kulkarni et~al.(2020)Kulkarni, Kulkarni, and Pant]{kulkarni2020survey}
V.~Kulkarni, M.~Kulkarni, and A.~Pant.
\newblock Survey of personalization techniques for federated learning.
\newblock In \emph{2020 fourth world conference on smart trends in systems,
  security and sustainability (WorldS4)}, pages 794--797. IEEE, 2020.

\bibitem[Lecun et~al.(1998)Lecun, Bottou, Bengio, and
  Haffner]{lecun_gradient-based_1998}
Y.~Lecun, L.~Bottou, Y.~Bengio, and P.~Haffner.
\newblock Gradient-based learning applied to document recognition.
\newblock \emph{Proceedings of the IEEE}, 86\penalty0 (11):\penalty0
  2278--2324, Nov. 1998.
\newblock ISSN 1558-2256.
\newblock \doi{10.1109/5.726791}.
\newblock Conference Name: Proceedings of the IEEE.

\bibitem[Li et~al.(2020)Li, Sahu, Zaheer, Sanjabi, Talwalkar, and
  Smith]{li2020federated}
T.~Li, A.~K. Sahu, M.~Zaheer, M.~Sanjabi, A.~Talwalkar, and V.~Smith.
\newblock Federated optimization in heterogeneous networks.
\newblock \emph{Proceedings of Machine learning and systems}, 2:\penalty0
  429--450, 2020.

\bibitem[Li et~al.(2021)Li, Hu, Beirami, and Smith]{li2021ditto}
T.~Li, S.~Hu, A.~Beirami, and V.~Smith.
\newblock Ditto: Fair and robust federated learning through personalization.
\newblock In \emph{International conference on machine learning}, pages
  6357--6368. PMLR, 2021.

\bibitem[Li et~al.(2019)Li, Huang, Yang, Wang, and Zhang]{li2019convergence}
X.~Li, K.~Huang, W.~Yang, S.~Wang, and Z.~Zhang.
\newblock On the convergence of fedavg on non-iid data.
\newblock \emph{arXiv preprint arXiv:1907.02189}, 2019.

\bibitem[Liang et~al.(2020)Liang, Liu, Ziyin, Allen, Auerbach, Brent,
  Salakhutdinov, and Morency]{liang2020think}
P.~P. Liang, T.~Liu, L.~Ziyin, N.~B. Allen, R.~P. Auerbach, D.~Brent,
  R.~Salakhutdinov, and L.-P. Morency.
\newblock Think locally, act globally: Federated learning with local and global
  representations.
\newblock \emph{arXiv preprint arXiv:2001.01523}, 2020.

\bibitem[Liu et~al.(2023{\natexlab{a}})Liu, Wu, Chen, Hu, Zhou, and
  Wu]{liu2023feddwa}
J.~Liu, J.~Wu, J.~Chen, M.~Hu, Y.~Zhou, and D.~Wu.
\newblock Feddwa: Personalized federated learning with dynamic weight
  adjustment.
\newblock \emph{arXiv preprint arXiv:2305.06124}, 2023{\natexlab{a}}.

\bibitem[Liu et~al.(2023{\natexlab{b}})Liu, Zhang, Li, Liu, and
  Yang]{liu2023dynamic}
Z.~Liu, Y.~Zhang, P.~Li, Y.~Liu, and D.~Yang.
\newblock Dynamic llm-agent network: An llm-agent collaboration framework with
  agent team optimization.
\newblock \emph{arXiv preprint arXiv:2310.02170}, 2023{\natexlab{b}}.

\bibitem[Mansour et~al.(2020)Mansour, Mohri, Ro, and Suresh]{mansour2020three}
Y.~Mansour, M.~Mohri, J.~Ro, and A.~T. Suresh.
\newblock Three approaches for personalization with applications to federated
  learning.
\newblock \emph{arXiv preprint arXiv:2002.10619}, 2020.

\bibitem[Marfoq et~al.(2021)Marfoq, Neglia, Bellet, Kameni, and
  Vidal]{marfoq2021federated}
O.~Marfoq, G.~Neglia, A.~Bellet, L.~Kameni, and R.~Vidal.
\newblock Federated multi-task learning under a mixture of distributions.
\newblock \emph{Advances in Neural Information Processing Systems},
  34:\penalty0 15434--15447, 2021.

\bibitem[McMahan et~al.(2017)McMahan, Moore, Ramage, Hampson, and
  Arcas]{mcmahan_communication-efficient_2017}
B.~McMahan, E.~Moore, D.~Ramage, S.~Hampson, and B.~A.~y. Arcas.
\newblock Communication-{Efficient} {Learning} of {Deep} {Networks} from
  {Decentralized} {Data}.
\newblock In \emph{Artificial {Intelligence} and {Statistics}}, pages
  1273--1282. PMLR, Apr. 2017.
\newblock ISSN: 2640-3498.

\bibitem[Muhammad et~al.(2020)Muhammad, Wang, O'Reilly-Morgan, Tragos, Smyth,
  Hurley, Geraci, and Lawlor]{muhammad2020fedfast}
K.~Muhammad, Q.~Wang, D.~O'Reilly-Morgan, E.~Tragos, B.~Smyth, N.~Hurley,
  J.~Geraci, and A.~Lawlor.
\newblock Fedfast: Going beyond average for faster training of federated
  recommender systems.
\newblock In \emph{Proceedings of the 26th ACM SIGKDD international conference
  on knowledge discovery \& data mining}, pages 1234--1242, 2020.

\bibitem[Ogier~du Terrail et~al.(2022)Ogier~du Terrail, Ayed, Cyffers,
  Grimberg, He, Loeb, Mangold, Marchand, Marfoq, Mushtaq, Muzellec,
  Philippenko, Silva, Tele\'{n}czuk, Albarqouni, Avestimehr, Bellet,
  Dieuleveut, Jaggi, Karimireddy, Lorenzi, Neglia, Tommasi, and
  Andreux]{terrail2022flamby}
J.~Ogier~du Terrail, S.-S. Ayed, E.~Cyffers, F.~Grimberg, C.~He, R.~Loeb,
  P.~Mangold, T.~Marchand, O.~Marfoq, E.~Mushtaq, B.~Muzellec, C.~Philippenko,
  S.~Silva, M.~Tele\'{n}czuk, S.~Albarqouni, S.~Avestimehr, A.~Bellet,
  A.~Dieuleveut, M.~Jaggi, S.~P. Karimireddy, M.~Lorenzi, G.~Neglia,
  M.~Tommasi, and M.~Andreux.
\newblock Flamby: Datasets and benchmarks for cross-silo federated learning in
  realistic healthcare settings.
\newblock In S.~Koyejo, S.~Mohamed, A.~Agarwal, D.~Belgrave, K.~Cho, and A.~Oh,
  editors, \emph{Advances in Neural Information Processing Systems}, volume~35,
  pages 5315--5334. Curran Associates, Inc., 2022.
\newblock URL
  \url{https://proceedings.neurips.cc/paper_files/paper/2022/file/232eee8ef411a0a316efa298d7be3c2b-Paper-Datasets_and_Benchmarks.pdf}.

\bibitem[Philippenko et~al.(2025)Philippenko, Scaman, and
  Massouli{\'e}]{philippenko2025depth}
C.~Philippenko, K.~Scaman, and L.~Massouli{\'e}.
\newblock In-depth analysis of low-rank matrix factorisation in a federated
  setting.
\newblock In \emph{Proceedings of the AAAI Conference on Artificial
  Intelligence}, volume~39, pages 19904--19912, 2025.

\bibitem[Robbins and Monro(1951)]{robbins1951stochastic}
H.~Robbins and S.~Monro.
\newblock A stochastic approximation method.
\newblock \emph{The annals of mathematical statistics}, pages 400--407, 1951.

\bibitem[Sattler et~al.(2020)Sattler, M{\"u}ller, and
  Samek]{sattler2020clustered}
F.~Sattler, K.-R. M{\"u}ller, and W.~Samek.
\newblock Clustered federated learning: Model-agnostic distributed multitask
  optimization under privacy constraints.
\newblock \emph{IEEE transactions on neural networks and learning systems},
  32\penalty0 (8):\penalty0 3710--3722, 2020.

\bibitem[T~Dinh et~al.(2020)T~Dinh, Tran, and Nguyen]{dinh2020personalized}
C.~T~Dinh, N.~Tran, and J.~Nguyen.
\newblock Personalized federated learning with moreau envelopes.
\newblock \emph{Advances in neural information processing systems},
  33:\penalty0 21394--21405, 2020.

\bibitem[Tan et~al.(2022)Tan, Yu, Cui, and Yang]{tan2022towards}
A.~Z. Tan, H.~Yu, L.~Cui, and Q.~Yang.
\newblock Towards personalized federated learning.
\newblock \emph{IEEE transactions on neural networks and learning systems},
  34\penalty0 (12):\penalty0 9587--9603, 2022.

\bibitem[Werner et~al.(2023)Werner, He, Jordan, Jaggi, and
  Karimireddy]{werner2023provably}
M.~Werner, L.~He, M.~Jordan, M.~Jaggi, and S.~P. Karimireddy.
\newblock Provably personalized and robust federated learning.
\newblock \emph{Transactions on Machine Learning Research}, 2023.
\newblock ISSN 2835-8856.
\newblock URL \url{https://openreview.net/forum?id=B0uBSSUy0G}.

\bibitem[Yu et~al.(2020)Yu, Bagdasaryan, and Shmatikov]{yu2020salvaging}
T.~Yu, E.~Bagdasaryan, and V.~Shmatikov.
\newblock Salvaging federated learning by local adaptation.
\newblock \emph{arXiv preprint arXiv:2002.04758}, 2020.

\end{thebibliography}

\newpage
\appendix

\begin{center}
	{\Large{\bf Supplementary material}}
\end{center}

In this appendix, we provide additional information to supplement our work. In \Cref{app:sec:experiments}, we provide complete details on our experiments and in \Cref{app:sec:proofs}, we give the proofs of our theorems. 
	
\setcounter{equation}{0}
\setcounter{figure}{0}
\setcounter{table}{0}
\setcounter{theorem}{0}
\setcounter{lemma}{0}
\setcounter{remark}{0}
\setcounter{proposition}{0}
\setcounter{property}{0}
\setcounter{definition}{0}

\renewcommand{\theequation}{S\arabic{equation}}
\renewcommand{\thefigure}{S\arabic{figure}}
\renewcommand{\thetheorem}{S\arabic{theorem}}
\renewcommand{\thelemma}{S\arabic{lemma}}
\renewcommand{\theproposition}{S\arabic{proposition}}
\renewcommand{\thecorollary}{S\arabic{corollary}}
\renewcommand{\thedefinition}{S\arabic{definition}}
\renewcommand{\theproperty}{S\arabic{property}}
\renewcommand{\theremark}{S\arabic{remark}}
\renewcommand{\thetable}{S\arabic{table}}

	\hypersetup{linkcolor = black}
	\setlength\cftparskip{2pt}
	\setlength\cftbeforesecskip{2pt}
	\setlength\cftaftertoctitleskip{3pt}
	\addtocontents{toc}{\protect\setcounter{tocdepth}{2}}
	\setcounter{tocdepth}{1}
	\tableofcontents
	
	\addtocontents{toc}{
    \protect\thispagestyle{empty}} 
    \thispagestyle{empty} 
    
    \hypersetup{linkcolor=blue}

\setlength{\parindent}{0pt}

\section{Experiments}
\label{app:sec:experiments}

This section provides additional details to support the experimental part of the main paper. We first describe the hardware configuration used to run all experiments, ensuring transparency of computational requirements. We then outline the setting of the synthetic dataset, followed by the procedure used to generate the synthetic splits of MNIST and CIFAR10 designed to simulate clients heterogeneity. Next, we summarize the complete training settings used across all experiments, including hyperparameters. Finally, we present complementary experimental results to give a more comprehensive view of model performance and personalization stability.

\paragraph{Hardware Configuration}
The experiments have been run on a machine equipped with two AMD EPYC 7302 processors (16 cores each, 3.0–3.3 GHz), a NVIDIA RTX 8000 GPU with 48 GB of memory, and 192 GB of system RAM. The experiments were conducted using three different seeds, each starting the training from scratch: 127 (Mersenne number), 496 (perfect number), 1729 (Ramanujan number).

\paragraph{Setting of the synthetic dataset.} We consider least-squares regression (LSR) with $N = 20$ clients. Each client $i \in {1, \dots, N}$ receives a pair of i.i.d. observations $(x_i, y_i) \sim \mathcal{D}_i$ in an online fashion, the feature dimension is fixed to $d \in \{2, 10\}$.. 
We assume the existence of two well-defined, client-specific model $\theta^{\star, 0}, \theta^{\star, 1} \in \mathbb{R}^d$ s.t.:
\begin{align*}
	\forall t \in \OneToT, \quad y_i^t = \PdtScl{x_i^t}{\theta^{\star, i\pmod 2 }}~, \quad \text{with}~~x_i^t \sim \mathcal{N}(0, H_i)\,, H_i = P_i^T \mathrm{Diag}(1) P_i\,, \text{and}\, P_i \sim \mathcal{O}_d\,, 
\end{align*}
therefore, we consider a setting where all clients inside a same cluster share the exact same optimal point.
For any $i \in \OneToN$, we consider the expected squared loss on client $i$ of a model $w$ as $R_{i}(\theta) := \E_{(x_i, y_i)\sim\mathcal{D}_i}[(\PdtScl{x_i}{\theta} - y_i)^2]$.
This setting is very close to the one considered in the experimental part of \citet{chayti2021linear}.
For $d=2$, we take a constant step-size $\gamma = (2\smoothness)^{-1}$, for $d=10$, and we take a reduced $\gamma = (4 \smoothness)^{-1}$, because the variance $\sigma_k$ increase with the dimensionality.

\paragraph{Setting of \texorpdfstring{\Cref{fig:sufficient_cluster}}{Figure 1}.} For this figure, we consider the setting where for each clients $i$ in $\OneToN$, there exist a true optimal model of the form $\theta^\star_i = \theta^\star + \xi_i$ in $\R^d$ with $\xi \sim \mathcal{N}(0, v^2 \mathrm{I}_d)$ and $v$ is in $\R_+$. This setup illustrate the concept of sufficient cluster, where clients initially collaborate with everyone and then shift into a localized training. The point of transition depends on the models' variance: the higher the variance, the higher the excess loss at which the collaboration regime shifts. The size of the sufficient clusters is computed by injecting \Cref{prop:quadratic_fn} in \Cref{def:effectif_cluster_var}.

\paragraph{Synthethic split of Mnist and Cifar10.} We split Mnist and Cifar10 accross the $N=20$ clients by using a cluster split defined below. This models statistical heterogeneity across clients.

\begin{definition}[Cluster-based split]
	\label{app:def:cluster_split}
	Let $(\mathbf{X}, \mathbf{Y})$ be a dataset with $K$ classes, such as MNIST or CIFAR-10. The cluster-based split partitions the dataset across $N$ clients based on classes, as follows:
	\begin{itemize}[leftmargin=*, itemsep=0pt, parsep=0pt, partopsep=0pt]
		\item Define two disjoint clusters of labels: $\mathcal{C}_0 = \{0, \dots, \floor{K/2}\}$ and $\mathcal{C}_1 = \{\floor{K/2} + 1, \dots, K\}$.
		\item Each client $i \in \OneToN$ is assigned data from a single cluster: clients with even index $i$ are assigned only examples with labels in $\mathcal{C}_0$, while odd-index clients receive data from $\mathcal{C}_1$.
		\item Within each cluster, data is partitioned uniformly across the $N/2$ clients assigned to it.
	\end{itemize}
	This results in a setting where no client observes the full label space,  the label distributions are disjoint across clients from different clusters, and local distributions are label-homogeneous within the assigned cluster.
\end{definition}

\begin{table}
	\caption{Settings of experiments}
	\label{app:tab:settings_xp}
	\centering
	\resizebox{\textwidth}{!}{%
		\begin{tabular}{>{\raggedright\arraybackslash}p{3.1cm}ccccc}
			\toprule
			Settings & Synth. & Mnist & Cifar10 & Heart Disease & Ixi \\
			\hline 
			references & NA & \citepalias{lecun_gradient-based_1998} & \citepalias{krizhevsky_learning_2009} & \citepalias{terrail2022flamby} & \citepalias{terrail2022flamby}\\
			task & linear regr. & \multicolumn{2}{c}{image classif.} & binary classif. & 3D segment.  \\
			model & Linear & CNN  & LeNet & Linear & Unet \\
			trainable parameters $d$ & $d \in \{2, 10 \}$ & $20\times 10^3$ & $62\times 10^3$ & $14$ & $246\times 10^3$ \\
			clients number $N$ & 20 & 20 & 20 & 4 & 3 \\
			split & NA & \multicolumn{2}{c}{cluster split} & \tcmv{natural} & \tcmv{natural} \\
			training dataset size & NA & $60,000$ & $60,000$ & $486$ & $453$ \\
			momentum $m$ & $0$ & $0$ & $0.9$ & $0$ & $0.9$ \\
			batch size $b$& $2$ & \multicolumn{2}{c}{$16$} & 1 & 8\\
			batch size $b_\alpha$ & $1$ & \multicolumn{2}{c}{$512$} & \multicolumn{2}{c}{$16$}\\
			inner iterations & 1 & \multicolumn{2}{c}{$50$} & \multicolumn{2}{c}{$\frac{1}{N} \sum_{i=1}^N n_i$} \\
			step size $\gamma$ & $\propto 1 / \smoothness$ & \multicolumn{2}{c}{$0.1$} & 0.05 & 0.01 \\ 
			StepLR scheduler & None & $(80, 10^{-1})$ & None & $(5, 10^{-1})$ & None \\
			Weight decay & $0$ & \multicolumn{4}{c}{$5\times 10^{-4}$} \\
			loss & mean square & \multicolumn{2}{c}{cross entropy} & logistic & dice \\
			\bottomrule
		\end{tabular}
	}
\end{table}

\begin{table}
	\centering
	\caption{Standard deviation of train and test accuracy/loss for the real datasets.}
	\label{app:tab:exp_real_std}
	\resizebox{\textwidth}{!}{%
		\begin{tabular}{p{0.09\textwidth}p{0.17\textwidth}cccccccc}
			\toprule
			& Algorithm & \multicolumn{2}{c}{Mnist} & \multicolumn{2}{c}{Cifar10} & \multicolumn{2}{c}{ Heart Disease} & \multicolumn{2}{c}{Ixi} \\
			& & Train & Test & Train & Test & Train & Test & Train & Test \\
			\midrule
			Accuracy & All-for-one-bin & 0.0 & 0.4 & 7.5 & 5.1 & 13.7 & 8.2 & 0.1 & 0.1 \\ 
			(in $\%$) & All-for-one-cont & 0.0 & 0.4 & 7.8 & 5.5 & 13.7 & 8.3 & 0.1 & 0.1 \\ 
			&  Local & 0.0 & 0.4 & 7.8 & 5.0 & 25.7 & 8.3 & 0.1 & 0.1 \\ 
			&  FedAvg & 1.2 & 0.5 & 27.5 & 28.2 & 22.0 & 5.5 & 1.8 & 2.0 \\ 
			&  Ditto & 0.0 & 0.4 & 6.8 & 4.5 & 22.0 & 8.4 & 0.1 & 0.1 \\ 
			&  Cobo & 0.0 & 0.4 & 7.6 & 5.1 & 16.6 & 8.5 & 0.1 & 0.1 \\ 
			&  Wga-bc & 1.8 & 0.8 & 7.2 & 6.9 & 12.6 & 5.0 & 2.6 & 2.6 \\ 
			&  Apfl & 0.0 & 0.4 & 8.4 & 5.0 & 17.5 & 8.4 & 4.8 & 5.0 \\ 
			Loss & All-for-one-bin & -4.01 & -1.64 & -0.11 & -0.04 & -0.67 & -0.41 & -1.57 & -1.56 \\ 
			(in log) & All-for-one-cont & -3.99 & -1.59 & -0.12 & -0.04 & -0.68 & -0.41 & -1.57 & -1.56 \\ 
			&  Local & -4.10 & -1.53 & -0.21 & 0.05 & -0.52 & -0.41 & -1.58 & -1.57 \\ 
			&  FedAvg & -1.61 & -1.40 & 0.26 & 0.26 & -0.22 & -0.26 & -1.08 & -1.08 \\ 
			&  Ditto & -3.86 & -1.52 & -0.22 & 0.06 & -0.41 & -0.40 & -1.58 & -1.57 \\ 
			&  Cobo & -4.01 & -1.53 & -0.21 & 0.05 & -0.60 & -0.41 & -1.58 & -1.57 \\ 
			&  Wga-bc & -0.94 & -0.91 & 0.36 & 0.36 & -0.12 & -0.16 & -1.20 & -1.19 \\ 
			&  Apfl & -3.94 & -1.54 & -0.22 & 0.05 & -0.43 & -0.41 & -0.53 & -0.54 \\ 
			\bottomrule
		\end{tabular}
	}
\end{table}

\begin{figure}[H]
	\centering
	\centering     
	\begin{subfigure}{0.24\linewidth}
		\includegraphics[width=\linewidth]{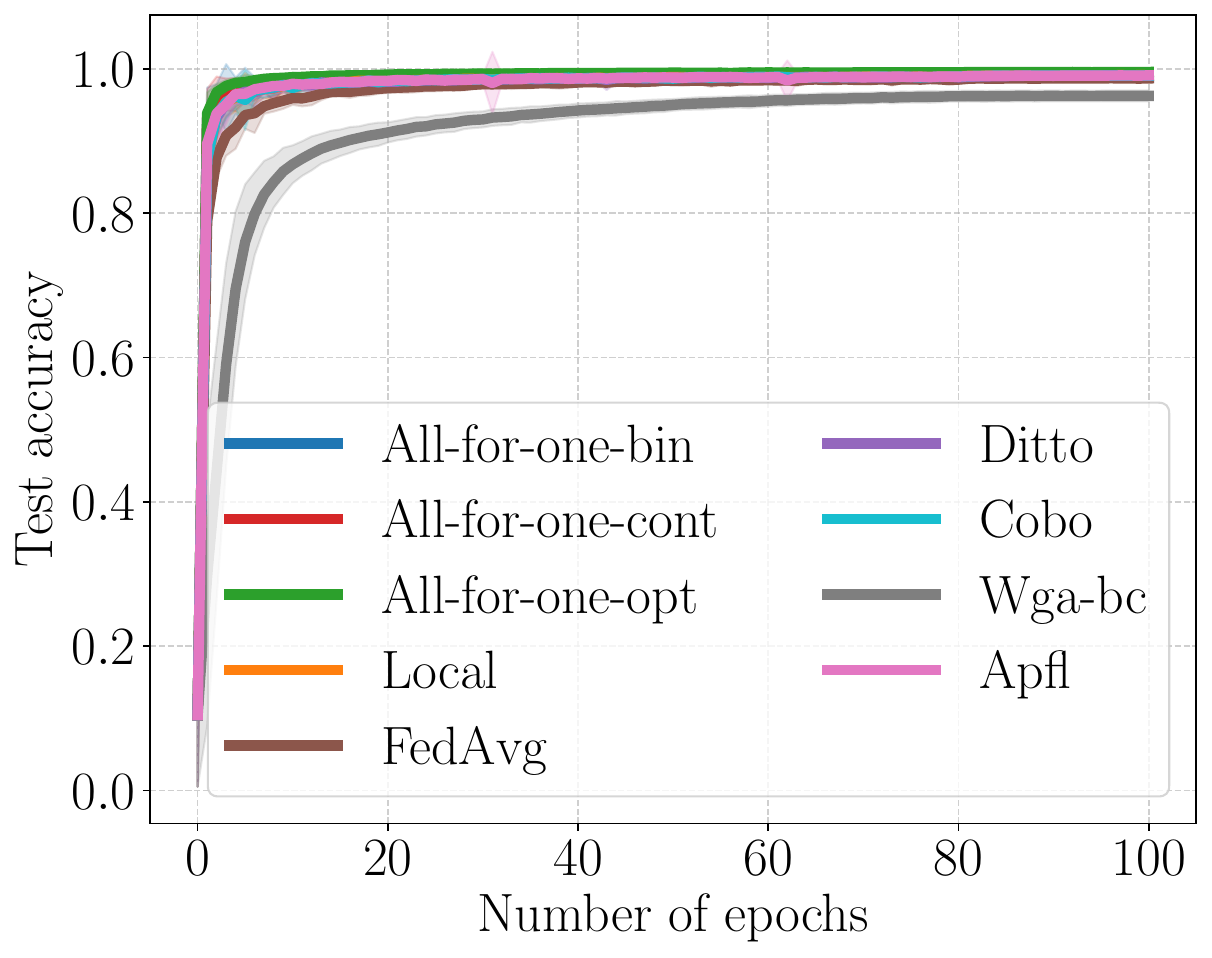}
		\caption{\label{app:fig:mnist} Mnist.} 
	\end{subfigure}
	\begin{subfigure}{0.24\linewidth}
		\includegraphics[width=\linewidth]{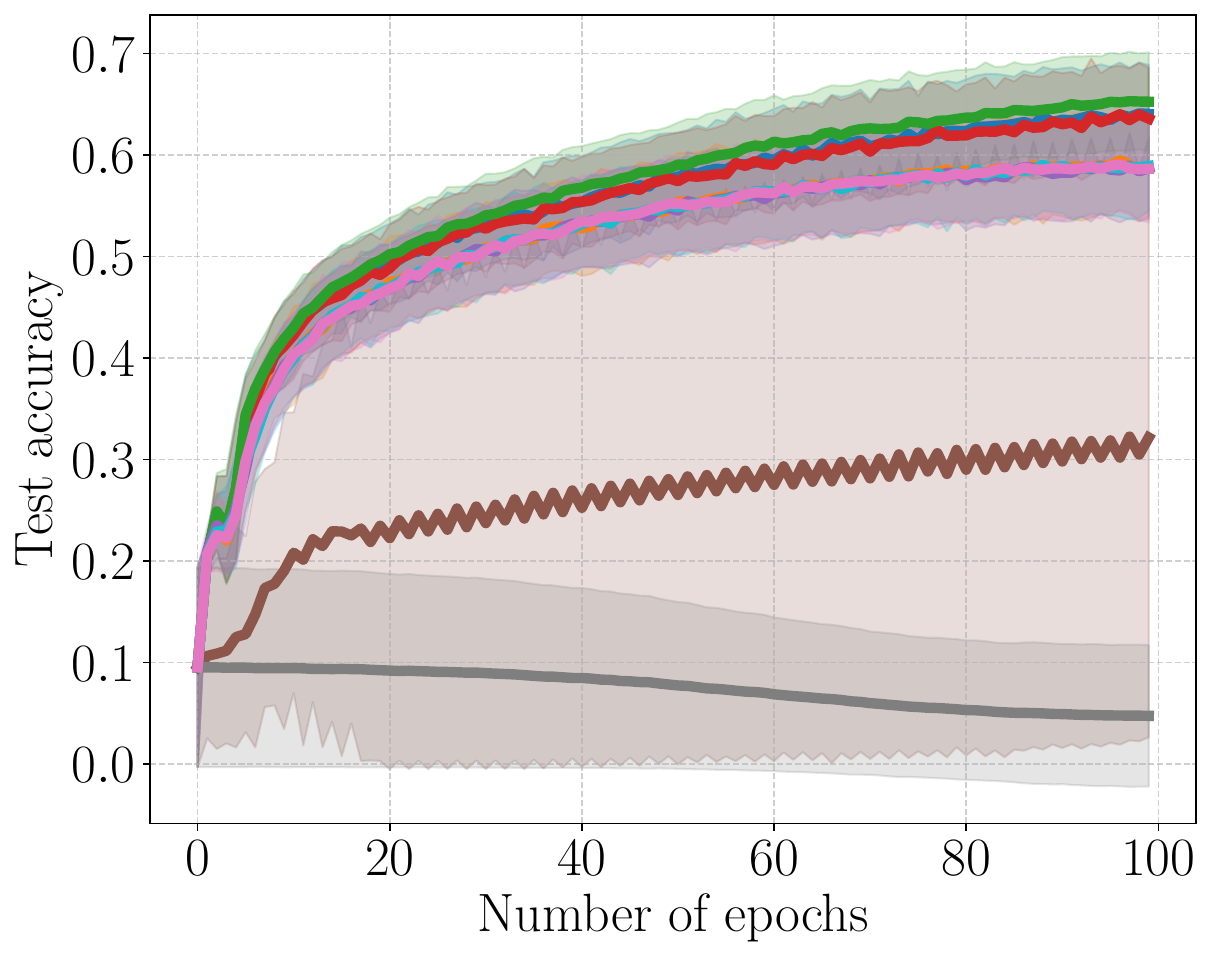}
		\caption{\label{app:fig:cifar10} Cifar10.}
	\end{subfigure}
	\begin{subfigure}{0.24\linewidth}
		\includegraphics[width=\linewidth]{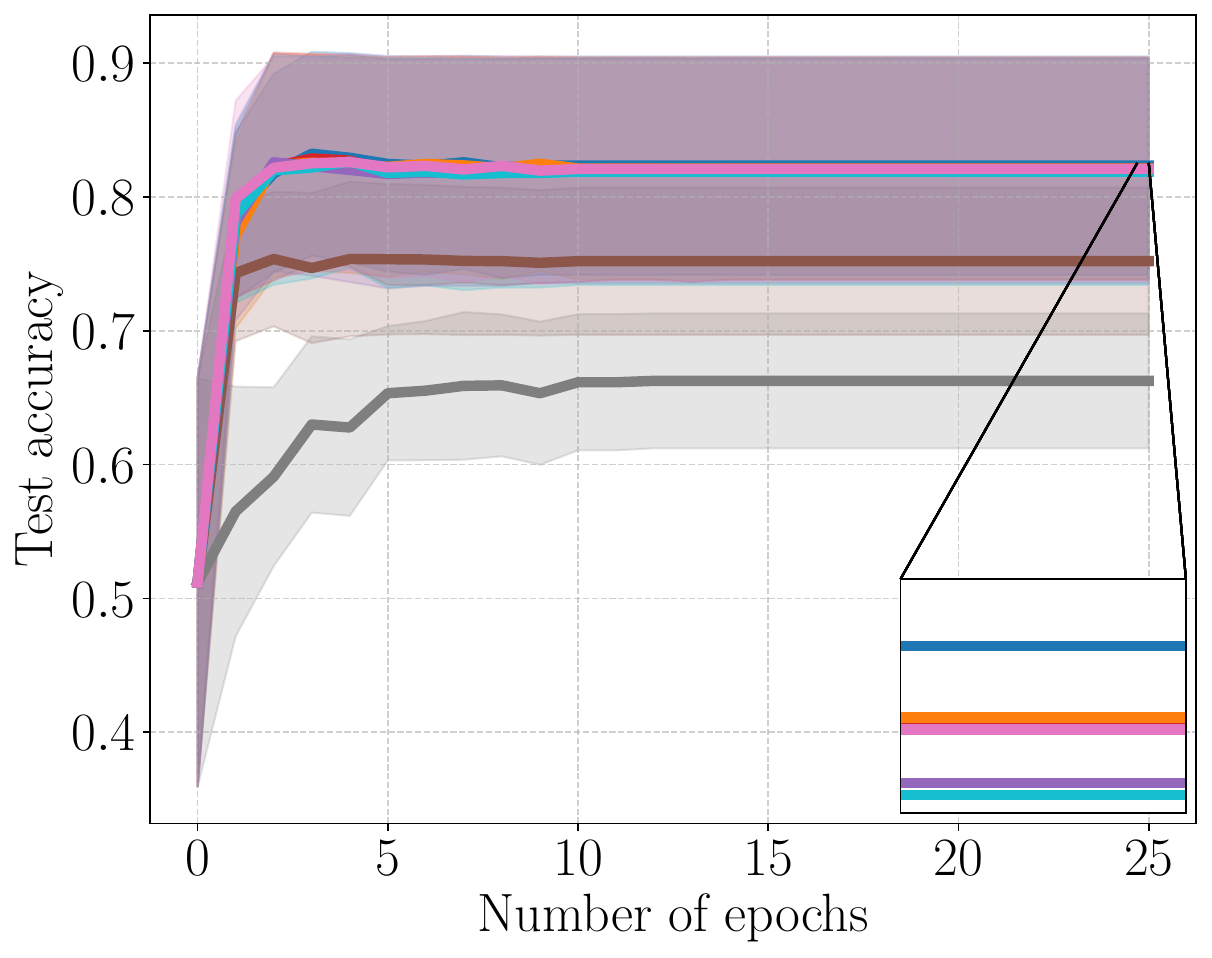}
		\caption{ \label{app:fig:heart_disease} Heart disease}
	\end{subfigure}
	\begin{subfigure}{0.24\linewidth}
		\includegraphics[width=\linewidth]{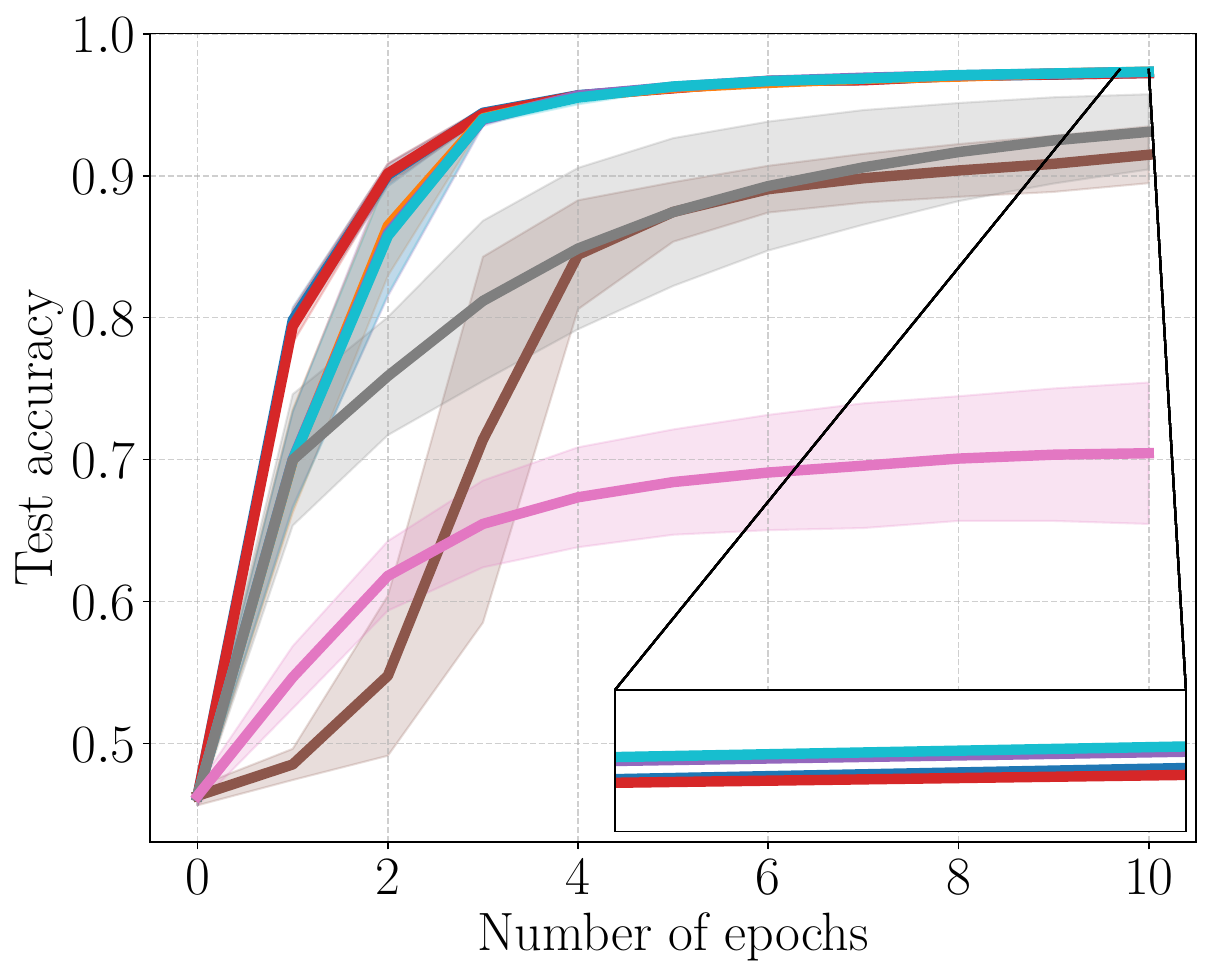}
		\caption{ \label{app:fig:ixi} Ixi}
	\end{subfigure}
	\caption{\label{app:fig:real_dataset} Test loss for real dataset. X-axis: number of epochs. Y-axis: test accuracy.}
\end{figure}

\paragraph{Settings of training.} \Cref{app:tab:settings_xp} summarizes the experimental configurations used across all datasets, including task types, models, and optimization hyperparameters. This ensures reproducibility of our experiments. The batch size $b_{\alpha}$ refers to the numbers of sample used to estimate the similarity ratio $(r_{ik})_{i,k \in \OneToN}$ as described in \Cref{sec:application}. The inner iterations denote the number of batch updates performed before recomputing the similarity weights, which helps reduce computational overhead. For Mnist and Cifar10, the weights are updated every 50 batches; and for Heart Disease and Ixi, the update occurs after each client has, on average, processed its entire local dataset once.

\paragraph{Complementary experimental results.} In \Cref{app:tab:exp_real_std}, we give the standard deviation of the train and test accuracy/loss for the four real datasets. In \Cref{app:fig:real_dataset}, we plot the average of test accuracy over the $N$ clients and the three seeds w.r.t. the number of of epoch. Note that the average is weighted by the relative size of each client.

\newpage

\section{Proofs}
\label{app:sec:proofs}

This Section gathers the full proofs of the theoretical results stated in the main paper.

\subsection{Proof of \texorpdfstring{\Cref{lem:descent_lemma}}{Lemma 1}}

\begin{proof}
\label{proof:descent_lemma}
    Fix a client $i \in \OneToN$ and an iteration index $t \in \N^*$. For clarity, we define $\mathcal{A}_i^t := \sum_{k=1}^N \alpha^t_{ik}$ and rewrite the \allforone~update as:
    $$
    \theta^t_i = \theta^{t-1}_i - \eta_i^t \mathcal{A}_i^t \sum_{k=1}^N \alpha_{ik}^{t} g_k^t(\theta^{t-1}_i) / \mathcal{A}_i^t\,.
    $$

    Let $g^t := \sum_{k=1}^N \alpha_{ik}^{t} g_k^t(\theta^{t-1}_i) / \mathcal{A}_i^t$ and $\nabla R_{\alpha_i}(\theta^{t-1}_i) := \sum_{k=1}^N \alpha_{ik}^t \nabla R_{k}(\theta^{t-1}_i) / \mathcal{A}_i^t$. By $\smoothness$-smoothness of $R_i$, we have:
    $$
    R_{i}(\theta^{t}_i) - R_{i}(\theta^{t-1}_i) \leq -\eta_i^t \mathcal{A}_i^t \PdtScl{ g^{t}}{\nabla R_{i}(\theta^{t-1}_i)} + \frac{(\eta_i^t \mathcal{A}_i^t)^2 \smoothness}{2} \sqrdnrm{g^{t}}\,.
    $$

    Noting that $R_i(\theta^t_i) - R_i(\theta_i^{t-1}) = \varepsilon_i^t - \varepsilon_i^{t-1}$, and taking expectation conditioned on $\theta^{t-1}_i$, we obtain:
    \begin{align}
    \label{app:eq:before_asu_var_grad_sto}
        \expec{\varepsilon_{i}^t}{\theta^{t-1}_i} - \varepsilon_{i}^{t-1}
        &\leq -\eta_i^t \mathcal{A}_i^t \PdtScl{\nabla R_{\alpha_i}(\theta^{t-1}_i)}{\nabla R_{i}(\theta^{t-1}_i)} 
        + \frac{(\eta_i^t \mathcal{A}_i^t)^2 \smoothness}{2} \Expec{\sqrdnrm{g^{t}}}{\theta^{t-1}_i}\,.
    \end{align}

    Since $\expec{g_k^t(\theta_i^{t-1})}{\theta_i^{t-1}} = \nabla R_k(\theta_i^{t-1})$ and that gradients across clients are independent given $\theta_i^{t-1}$, we have:
    \begin{align*}
        \expec{\sqrdnrm{g^t}}{\theta^{t-1}_i}
        &= \expec{\sqrdnrm{\sum_{k \in \OneToN} \ffrac{\alpha_{ik}^t}{\mathcal{A}_i^t} g_k^t(\theta^{t-1}_i)}}{\theta^{t-1}_i} \\
        &= \expec{\sqrdnrm{\sum_{k \in \OneToN} \ffrac{\alpha_{ik}^t}{\mathcal{A}_i^t} (g_k^t(\theta^{t-1}_i) - \nabla R_k(\theta_i^{t-1}))}}{\theta^{t-1}_i} 
        + \sqrdnrm{\nabla R_{\alpha_i}(\theta_i^{t-1})} \\
        &= \sum_{k \in \OneToN} \bigpar{\ffrac{\alpha_{ik}^t}{\mathcal{A}_i^t}}^2 \expec{\sqrdnrm{g_k^t(\theta^{t-1}_i) - \nabla R_k(\theta_i^{t-1})}}{\theta^{t-1}_i} 
        + \sqrdnrm{\nabla R_{\alpha_i}(\theta_i^{t-1})}\,.
    \end{align*}

    Next with \Cref{asu:bounded_variance} and back to \Cref{app:eq:before_asu_var_grad_sto}, we have:
    \begin{align*}
        \expec{\varepsilon_{i}^t}{\theta^{t-1}_i} - \varepsilon_{i}^{t-1}
        &\leq -\eta_i^t \mathcal{A}_i^t \PdtScl{\nabla R_{\alpha_i}(\theta^{t-1}_i)}{\nabla R_{i}(\theta^{t-1}_i)} \\
        &\qquad + \frac{(\eta_i^t \mathcal{A}_i^t)^2 \smoothness}{2} \left( \sqrdnrm{\nabla R_{\alpha_i}(\theta_i^{t-1})} + \sum_{k=1}^N \ffrac{(\alpha_{ik}^{t})^2 \sigma_k^2}{(\mathcal{A}_i^t)^2} \right)\,.
    \end{align*}

    Using the identity $\pdtscl{a}{b} = \frac{1}{2} \left( \sqrdnrm{a} + \sqrdnrm{b} - \sqrdnrm{a - b} \right)$, we obtain:
    \begin{align*}
        \expec{\varepsilon_{i}^t}{\theta^{t-1}_i} - \varepsilon_{i}^{t-1}
        &\leq -\frac{\eta_i^t \mathcal{A}_i^t}{2} (1 - \eta_i^t \mathcal{A}_i^t \smoothness) \sqrdnrm{\nabla R_{\alpha_i}(\theta^{t-1}_i)} 
        - \frac{\eta_i^t \mathcal{A}_i^t}{2} \sqrdnrm{\nabla R_{i}(\theta^{t-1}_i)} \\
        &\qquad + \frac{\eta_i^t \mathcal{A}_i^t}{2} \sqrdnrm{\nabla R_{\alpha_i}(\theta^{t-1}_i) - \nabla R_{i}(\theta^{t-1}_i)} 
        + \frac{(\eta_i^t)^2 \smoothness \sum_{k=1}^N \sigma_k^2 (\alpha_{ik}^{t})^2}{2}\,.
    \end{align*}

    Assuming $1 - \eta_i^t \mathcal{A}_i^t \smoothness > 0$, we can drop the first term. Also, using $\mathcal{A}_i^t = \sum_{k=1}^N \alpha_{ik}^t$, we deduce:
    \begin{align*}
        \expec{\varepsilon_{i}^t}{\theta^{t-1}_i} - \varepsilon_{i}^{t-1}
        &\leq \frac{\eta_i^t}{2} \sum_{k=1}^N \alpha_{ik}^{t} \left( \sqrdnrm{\nabla R_{k}(\theta^{t-1}_i) - \nabla R_{i}(\theta^{t-1}_i)} - \sqrdnrm{\nabla R_{i}(\theta^{t-1}_i)} \right) \\
        &\qquad + \eta_i^t \smoothness \sum_{k=1}^N (\alpha_{ik}^{t})^2 \sigma_k^2\,,
    \end{align*}
    which concludes the proof.
\end{proof}

\subsection{Proof of 
\texorpdfstring{\Cref{thm:convergence,thm:convergence_non_cvx,thm:descent_lemma_with_N_i_t}}{Theorem 1 to 3}}
\label{app:subsec:proof_cvgce}

\begin{proof}

For any function $f : R \mapsto R$, we recall that we denote $\sigma_{i, f}^t \coloneqq \bigpar{\sum_{k} \frac{1}{\sigma_k^2}  f(r_{ik})}^{-1/2}$ and that by \Cref{def:weights_binary}:
\[
\left\{
\begin{aligned}
    r_{ik}^t &= \frac{\sqrdnrm{\nabla R_{i}(\theta^{t-1}_i)} - \sqrdnrm{\nabla R_{k}(\theta^{t-1}_i) - \nabla R_{i}(\theta^{t-1}_i)}}{\sqrdnrm{\nabla R_{i}(\theta^{t-1}_i)}} \,,\\
    \alpha_{ik}^t &= \frac{(\sigma_{i, \psi}^t)^2}{\sigma_k^2} \phi(r_{ik}^t) \,, \\
    \psi(r_{ik}^t) &= r_{ik}^t \phi(r_{ik}^t) \,.
\end{aligned}
\right.
\]

From this definitions, we first deduce: 
    \begin{align*}
        \sum_{k=1}^N \alpha_{ik}^{t} \bigpar{\sqrdnrm{\nabla R_{k}(\theta^{t-1}_i) - \nabla R_{i}(\theta^{t-1}_i)} - \sqrdnrm{\nabla R_{i}(\theta^{t-1}_i)}} &= - \sum_{k=1}^N \frac{(\sigma_{i, \psi}^t)^2}{\sigma_k^2} \phi(r_{ik}^t) r_{ik} \sqrdnrm{\nabla R_i(\theta_i^{t-1})} \\
        &= - \sqrdnrm{\nabla R_i(\theta_i^{t-1})} \,.
    \end{align*}
    This is a central aspect of our proof justifying the particular form of the collaboration weights (\Cref{def:weights_binary}). Second, we have:
    \begin{align*}
        &\sum_{k=1}^N (\alpha_{ik}^{t} \sigma_{k})^2 =  \sum_{k=1}^N \frac{(\sigma_{i, \psi}^t)^4}{\sigma_k^4} \phi^2(r_{ik}^t) \sigma_k^2 = \frac{(\sigma_{i, \psi}^t)^4}{(\sigma_{i, \phi^2}^t)^2}  \leq (\sigma_{i, \psi}^t)^2 = (\sigma_{i, \mathrm{eff}}^t)^2 \,.
    \end{align*}
    Where the last inequality comes from the fact that $\phi(x)\leq x$, yielding $\phi^2\leq \psi$.
    It results that \Cref{lem:descent_lemma} can be rewritten as follows:
    \begin{align}
    \label{eq:before_mu}
        \Expec{\varepsilon_{i}^t}{\theta^{t-1}_i} - \varepsilon_{i}^{t-1} \leq \frac{-\eta_i^t }{2} \sqrdnrm{\nabla R_{i}(\theta^{t-1}_i)} +\frac{(\eta_i^t \sigma_{i, \mathrm{eff}}^t)^2 \smoothness}{2} \,.
    \end{align}
    Taking full expectation, it gives in the strongly-convex setting (\cref{asu:strongly_convex}) or with the PL-condition (\Cref{asu:PL}):
    \begin{align*}
        \fullexpec{\varepsilon_{i}^t}) \leq \bigpar{1 -\eta_i^t \mu } \fullexpec{\varepsilon_{i}^{t-1}} +\frac{(\eta_i^t\sigma_{i, \mathrm{eff}}^t)^2 \smoothness}{2}\,.
    \end{align*}

    If we take a constant step-size $\eta_i^t = \eta_i$, summing over $T$ gives:
    \begin{align*}
        \FullExpec{\varepsilon_{i}^T} \leq \bigpar{1 - \eta_i \mu  }^T \varepsilon_{i}^0 + \frac{\eta_i^2 \smoothness}{2} \sum_{t=1}^T (1-\eta_i\mu)^{T-t} (\sigma_{i, \mathrm{eff}}^t)^2\,,
    \end{align*}
    which proves \Cref{thm:descent_lemma_with_N_i_t} recalling that $\sigma_{i, \mathrm{eff}}^t \coloneqq \sigma_{i, \psi}^t$ and that $\sum_{i \in \OneToN} \alpha_{ik} = (\sigma_{i, \psi}^t/ \sigma_{i, \phi}^t)^2$. 
    
    We now prove~\Cref{thm:convergence}. Upper-bounding the previous result with the minimal variance of the sufficient cluster~$\mathcal{N}_i^*(\varepsilon)$ (\Cref{lem:nested_set_of_collaborating_clients}), we have: 
    \begin{align*}
        \FullExpec{\varepsilon_{i}^T} \leq \bigpar{1 - \eta \mu }^T \varepsilon_{i}^0 + \frac{\eta \smoothness}{2 \mu} \sigma_{i,{\mathrm{suf}}}^2(\varepsilon)\,,
    \end{align*}
    with a bias term contracting for a constant step-size $\eta$ s.t. $\eta \leq \frac{1}{\mu}$. To obtain the optimal horizon-dependent step-size, we use the inequality $ 1 - x \leq e^{-x}$ for any $x$ in $\R$, and optimize this upper-bound w.r.t. $\eta$ to obtain $\eta_i = \frac{1}{\mu  T} \ln\bigpar{\frac{2\mu^2 \varepsilon_i^0 T}{\smoothness \sigma_{i,{\mathrm{suf}}}^2(\varepsilon )}} \leq \frac{1}{\mu}$, which gives: 
     \begin{align*}
         \FullExpec{\varepsilon_{i}^T} \leq \frac{\smoothness \sigma_{i,{\mathrm{suf}}}^2(\varepsilon)}{2 \mu^2 T} \bigpar{ 1 + \ln\bigpar{\frac{2 T \mu^2 \varepsilon_{i}^0}{\sigma_{i,{\mathrm{suf}}}^2(\varepsilon)}}} \,.
     \end{align*}

    For the decreasing step-size $\eta_i^t = \frac{C}{\mu  t}$ with $C > 1$, we  come back to \Cref{eq:before_mu}. As it is the same than Equation 4.23 in \citep{bottou2018optimization}, their Theorem 4.7 can be applied if we upper-bound the effective variance with the sufficient variance, and it results to:
    \begin{align*}
         \FullExpec{\varepsilon_{i}^T} \leq \frac{1}{T} \max(\frac{\smoothness \sigma_{i, \mathrm{bin}}^2(\varepsilon ) C^2}{2 \mu^2( C - 1)}, \varepsilon_{i}^0)\,.
    \end{align*}

    Altogether, we have proven \Cref{thm:convergence}.
    And finally to prove \Cref{thm:convergence_non_cvx} (non-convex setting without assuming the PL-condition), back to \Cref{eq:before_mu}, upper-bounding again with the minimal variance of cluster $\mathcal{N}_i^*(\epsilon)$, considering a constant step-size $\eta_i^t = \eta_i$, and unrolling all the iterations, we have:
    \begin{align*}
         \frac{1}{T} \sum_{t=1}^{T} \FullExpec{\sqrdnrm{\nabla R_{i}(\theta^{t-1}_i)}} \leq \frac{2 \varepsilon_{i}^0}{T \eta_i } + \eta_i \smoothness \sigma_{i,{\mathrm{suf}}}^2(\varepsilon)\,.
    \end{align*}    

    Minimizing the RHS w.r.t. $\eta_i$ gives $\eta_i^{\mathrm{opt}} = \sqrt{\frac{2 \varepsilon_{i}^0}{T \smoothness \sigma_{i,{\mathrm{bin}}}^2(\varepsilon)}}$ and thereby:
    \begin{align*}
         \frac{1}{T} \sum_{t=1}^{T} \FullExpec{\sqrdnrm{\nabla R_{i}(\theta^{t-1}_i)}} \leq 2 \sqrt{\frac{2 \varepsilon_{i}^0 \smoothness \sigma_{i,{\mathrm{bin}}}^2(\varepsilon)}{T}} \,,
    \end{align*}    
    which proves \Cref{thm:convergence_non_cvx}.
\end{proof}

\subsection{Proof of 
\texorpdfstring{\Cref{cor:sampling_complexity_optimality}}{Corollary 1}}
\label{app:subsec:proof_sampling_complexity}

\begin{proof}
Let any precision level $\varepsilon>0$, then after running \allforone~for   $T$ iterations in $\N^\star$ s.t. $\varepsilon_i^{T-1}\geq \varepsilon$, in the case of a decreasing step-size as defined in \Cref{tab:summary_res_convergence} (\Cref{thm:convergence}), we consider that the upper-bound on the excess loss is:
\begin{align*}
\fullexpec{\varepsilon_i^T} \leq \ffrac{1}{T} \cdot \ffrac{\smoothness \sigma_{i, \mathrm{suf}}^2(\varepsilon) C^2}{2 \mu^2 ( C - 1)} \,.
\end{align*}

We denote $\mathcal{C} = C^2 / (C - 1)$, and we look for $T_i^\varepsilon$ such that the above upper-bound is equal to $\varepsilon$, i.e., $
\ffrac{\smoothness \sigma_{i, \mathrm{suf}}^2(\varepsilon)\mathcal{C}}{2 \mu^2 T_i^\varepsilon} = \varepsilon$, it naturally yields $T_i^\varepsilon = \ffrac{\smoothness \sigma_{i, \mathrm{suf}}^2(\varepsilon) \mathcal{C}}{2 \mu^2 \varepsilon}
$.

To establish a connection with the lower bound provided by \citet{even2022sample}, we now introduce two additional assumptions to align our setting with theirs. First, assume that the variance of the stochastic gradients is uniform across clients, i.e., there exists $\sigma \in \R_+^\star$ such that for all $k \in \OneToN$, $\sigma_k = \sigma$. Second, assume that for any pair of clients $i, k$ in $\OneToN$, we have $c_{ik} = 0$. Under these assumptions, we obtain:
\[
\sigma_{i, \mathrm{suf}}^2(\varepsilon) = \sigma^{2} \bigpar{\sum_{k \in \OneToN}  \psi\bigpar{1 - \ffrac{b_{ik}}{2 \mu \varepsilon}} }^{-1} \leq \ffrac{\sigma^2}{\mathcal{N}_i^\star(\varepsilon) \min_{j \in \mathcal{N}_i^\star(\varepsilon)} \psi\bigpar{1 - \ffrac{b_{ik}}{2 \mu \varepsilon}}} \,,
\]
which concludes the proof.
\end{proof}

\end{document}